\documentclass[authoryear,3p,times]{elsarticle}

\usepackage{natbib}
\usepackage{amssymb}
\usepackage{amsmath}
\usepackage{bm}

\usepackage[utf8]{inputenc} 
\usepackage[T1]{fontenc}    
\usepackage{hyperref}       
\usepackage{url}            
\usepackage{booktabs}       
\usepackage{amsfonts}       
\usepackage{nicefrac}       
\usepackage{microtype}      

\usepackage{subcaption}
\usepackage{lipsum}
\usepackage{graphicx}
\graphicspath{ {./images/} }
\usepackage{xcolor}
\usepackage{float}

\usepackage[colorinlistoftodos,prependcaption,textsize=small]{todonotes}

\usepackage{amsthm}
\newtheorem{proposition}{Proposition}

\newtheorem{definition}{Definition}

\journal{Transportation Research Part B: Methodological}

\begin{document}

\begin{frontmatter}

\title{A Semiparametric Framework for Stochastic Fundamental Diagram Modeling}

\author[add1]{Pengnan Chi}
\ead{pengnan@kth.se} 

\author[add1]{Xiaoliang Ma}
\ead{liang@kth.se}

\author[add1]{Magnus Jansson}
\ead{janssonm@kth.se}

\author[add3]{Magnus Nordenvaad}
\ead{magnus.nordenvaad@trafikverket.se}

\affiliation[add1]{
    organization={KTH Royal Institute of Technology},
    city={Stockholm},
    country={Sweden}
}

\affiliation[add3]{
    organization={Swedish Transport Administration},
    city={Borlänge},
    country={Sweden}
}

\begin{abstract} 
The stochastic fundamental diagram (SFD) provides a probabilistic description of the relationship between traffic density and flow or speed, enabling uncertainty-aware traffic modeling. 
However, existing stochastic models frequently struggle to accommodate rigorous physical constraints while retaining sufficient flexibility to capture complex nonlinear patterns.
To address this, we propose a novel semiparametric SFD modeling framework by leveraging specially designed functional forms. These functions intrinsically satisfy physical constraints defined on the moments of the conditional flow distribution given traffic density while incorporating neural-network-based structures to capture complex empirical patterns.
We derive a system of moment-matching equations to convert physical constraints into the parameterization of the conditional distribution, proving that a unique solution exists for the location–scale family of distributions, thereby guaranteeing model well-posedness. 
Furthermore, we demonstrate that the framework can be extended to non-location–scale distributions, including those requiring additional boundary constraints. Empirical evaluations on the real-world dataset reveal that our approach consistently outperforms representative baselines, delivering superior probabilistic accuracy and robust uncertainty quantification, particularly in congested regimes. Overall, the proposed framework provides a theoretically grounded and flexible foundation for stochastic traffic flow modeling.
\end{abstract}


\end{frontmatter}

\section{Introduction}
The fundamental diagram (FD) describes the intrinsic relationship among three macroscopic traffic variables, i.e., traffic density $\rho$, flow rate $q$, and space-mean speed $v$, under steady-state conditions.  Since traffic flow is the product of speed and density, the fundamental diagram is typically represented by either a speed-density or flow-density relationship. 
Studies on the fundamental diagram provide critical insights on the vehicle movement patterns on a road segment, enabling the development of models for traffic simulation, congestion detection, and traffic management strategies. 

Classical FD models typically assume a deterministic relationship between traffic density and speed, where each density corresponds to a uniquely speed value. The seminal work of the Greenshields model \citep{greenshields1935study} establishes a linear speed-density relationship, followed by several refinements such as the Greenberg model \citep{greenberg1959analysis} and the Newell model \citep{newell1961nonlinear}. Although these classical models have provided valuable insights into traffic flow theory, they inherently overlook the stochastic fluctuations and heterogeneous driver behaviors that characterize traffic in real world. This limitation has motivated the development of stochastic approaches to FD modeling, which aims to capture inherent uncertainty present in real-world traffic observations. 

To better capture the random characteristics inherent in road traffic flow, the stochastic fundamental diagram (SFD) has been proposed. Instead of mapping density to a single deterministic flow value, the SFD characterizes a probability distribution of flow or speed for a given density. Existing SFD models, however, are often closely tied to deterministic FD formulations. 
One common approach treats the parameters of a deterministic model, such as free-flow speed and jam density, as random variables that follow a prescribed distribution \citep{jabari2014probabilistic, zhou2020modeling, bai2021calibration, cheng2024analytical}. 
Another prevalent approach is to integrate deterministic FD structures with stochastic processes to represent random traffic dynamics \citep{ni2018modeling, lei2024unraveling, cheng2022bayesian}.

While conceptually well-defined, these approaches inherit limitations from the deterministic models on which they are built. 
In particular, widely used single-regime deterministic models often have limited capacity to represent complex traffic states and transitions across the full range of observed densities \citep{ni2015traffic}. Consequently, relying on predetermined deterministic forms in certain traffic regimes may introduce systematic bias that undermines modeling accuracy.
In addition, there is no consensus on the appropriate deterministic formulation to adopt in SFD modeling \citep{lei2024unraveling, cheng2022bayesian, wang2021model}. This lack of agreement stems from the different underlying assumptions adopted by existing models, making it difficult to establish clear and unified guidelines for model selection.

These limitations motivate the development of a modeling framework with enhanced nonlinear modeling capabilities as well as strict adherence to traffic flow physics. In this study, we adopt a semiparametric modeling strategy that integrates physically motivated constraints with data-driven components. Rather than relying on a predefined deterministic fundamental diagram, the proposed approach enforces rigorously-defined physical properties while allowing functional relationships to be learned from data. 
Specifically, a semiparametric formulation is constructed in which physically constrained structures ensure consistency with traffic flow theory, while neural network (NN) based correction terms are introduced to capture complex patterns observed in empirical traffic data. In addition, traffic flow conditioned on density is modeled through a probabilistic distribution capable of representing asymmetric variability commonly observed in real traffic conditions. The distribution parameters are modeled as semiparametric functions of density, enabling the model to represent the full conditional distribution of traffic flow across different density regimes.


The main contributions of this paper are summarized as follows:
\begin{itemize}
    \item A semiparametric modeling framework is proposed for estimating the stochastic flow-density relationship in the traffic fundamental diagram, which enforces physical constraints while flexibly capturing complex patterns in empirical data.
    \item A physically constrained parameterization strategy is developed for selected probabilistic distributions; the strategy is proven to be valid for the location-scale family and demonstrates strong extensibility to other complex distributions.
    \item The proposed approach is validated using extensive experiments on real-world traffic data, with appropriate evaluation workflow to assess modeling accuracy and uncertainty quantification.
\end{itemize}


\section{Related Works}
\subsection{Deterministic Fundamental Diagrams}
Traditionally, deterministic fundamental diagrams characterize traffic flow primarily through the speed--density relationship. The Greenshields model \citep{greenshields1935study}, which assumes a linear relationship between speed and density, was the first to formalize this concept. Since then, extensive research has sought to improve the accuracy and interpretability of deterministic fundamental diagram models. Based on their functional forms, these models are commonly classified into three major categories: logarithmic, exponential, and polynomial formulations.
The detailed mathematical expressions for representative models in each category are summarized in Table \ref{tab:det_list}.
\begin{table}[t]
    \centering
    \begin{tabular}{ccc}
    \toprule
    Category  &  Model & Formulation \\
    \midrule
    Polynomial function-based &  Greenshields \cite{greenshields1935study}  & $v = v_{\mathrm{free}}(1-\frac{\rho}{\rho_{jam}})$ \\[2ex]
    & Jayakrishnan \cite{jayakrishnan1995dynamic} & $v=v_{\mathrm{min}}+(v_{\mathrm{free}}-v_{\mathrm{min}})\left(1-\frac{\rho}{\rho_{\mathrm{jam}}}\right)$ \\[2ex]
    & S3 \cite{cheng2021s} & $v = \frac{v_{\mathrm{free}}}{\left[1+\left(\frac{\rho}{\rho_{\mathrm{critical}}}\right)^m\right]^{\frac{2}{m}}} $\\
    \midrule
    Exponential function-based & Newell\cite{newell1961nonlinear} & $v = v_{\mathrm{free}}\left\{1-\exp\left[-\frac{\lambda}{v_{\mathrm{free}}} \left( \frac{1}{\rho} - \frac{1}{\rho_{\mathrm{jam}}}  \right)  \right] \right\} $\\[2ex]
    & Papageorgiou \cite{papageorgiou1989macroscopic} & $v=v_{\mathrm{free}}\exp\left[-\frac{1}{\alpha} \left( \frac{\rho}{\rho_{\mathrm{jam}}}\right)^\alpha \right]$\\[2ex]
    & Del Castillo \cite{del1995functional} & $v=v_{\mathrm{free}} \left\{1-\exp\left[\frac{v_{\mathrm{jam}}}{v_{\mathrm{free}}}\left(1-\frac{\rho_{\mathrm{jam}}}{\rho} \right) \right] \right\} $\\[2ex]
     & Wang \cite{wang2011logistic} & $v=v_{\mathrm{critical}} + \frac{v_{\mathrm{free}} - v_{\mathrm{critical}}}{\left[ 1 + \exp\left( \frac{\rho-\rho_{\mathrm{critical}}}{\theta_1} \right) \right]^
     {\theta_2}} $\\
     \midrule
     Logarithmic function-based & Greenberg \cite{greenberg1959analysis}& $v=v_{\mathrm{critical}}\log\left( \frac{\rho_{\mathrm{jam}}}{\rho}\right)$\\
    \bottomrule
    \end{tabular}
    \caption{Three Categories of Deterministic Models}
    \label{tab:det_list}
\end{table}

\subsection{Stochastic Fundamental Diagrams}



The most common approach to modeling SFD is to extend existing deterministic models by introducing random parameters. In this framework, selected parameters of deterministic models are treated as random variables to capture variability in steady-state traffic conditions. 
For example, \citet{jabari2014probabilistic} and \citet{zhou2020modeling} modeled microscopic parameters such as headway and speed adaptation time as random variables and derived the macroscopic fundamental diagram relations under steady-state assumptions.
Similarly, \citet{qu2017stochastic} and \citet{wang2021model} estimated model parameters at different distribution quantiles, resulting in discrete stochastic representations of the fundamental diagram.
model with random variables and obtained the stochastic speed-density relationship.
Building on this approach, \citet{ahmed2021fundamental} further analyzed traffic flow heterogeneity arising from different vehicle types. In addition to randomizing existing model parameters, \citet{bai2021calibration} introduced additional random variables to explicitly represent traffic heterogeneity. \cite{cheng2024analytical} investigated analytical properties of the model parameter when making other parameters as random variables.

In addition to random parameter formulations, another commonly adopted approach is to treat a deterministic fundamental diagram as the expected value and introduce stochasticity through an explicit random process. In this framework, the deterministic model describes the mean traffic behavior while stochastic variability is modeled separately. 
For example, \citet{ni2018modeling} applied a complex deterministic model \citep{ni2016vehicle} for the expected values and characterized the stochasticity using the Maxwell–Boltzmann distribution inspired by ideal gas theory. 
Similarly, \citet{liu2023gaussian} extend the deterministic fundamental diagrams using Gaussian process regression to incorporate additional traffic variables.
Along the same line, \citet{lei2024unraveling}and \citet{cheng2022bayesian} adopted the deterministic models as the expectation function and introduced stochasticity with Gaussian process regression.

Some studies depart entirely from deterministic fundamental diagram formulations. Instead, they derive stochastic traffic relationships directly from traffic flow dynamics or alternative stochastic representations.
For example, several studies model traffic evolution based on the continuity equation to describe macroscopic traffic dynamics without relying on predefined fundamental diagrams \citep{shi2021physics, storm2022efficient, ngoduy2011multiclass}. In a different line of study, \citet{zhang2025stochastic} and \citet{zhang2025stochastic} described the platoon behavior using Markov chains and derived the macroscopic stochastic FD. Moreover, \citet{bramich2023fitfun} employed an advanced statistical regression framework, namely generalized additive models for location, scale, and shape (GAMLSS) \citep{rigby2005generalized}, to develop the SFD model. 

\section{Methodology}
\subsection{Problem Definition}
\label{sec:problem_definition}
This study treats traffic flow as a random variable conditioned on density. Let $\rho \in [0, \rho_{\mathrm{jam}}]$ and $q \in \mathbb{R}_{\ge 0}$ denote the traffic density and traffic flow, respectively, where $\rho_{\mathrm{jam}}$ is the traffic jam density. The objective is to model the
conditional probability distribution of traffic flow given density, denoted by $p(q \mid \rho)$.
We assume that, for each fixed density $\rho$, the conditional distribution $p(q \mid \rho)$ admits an analytical probability density function (PDF) belonging to a parametric family $f_{\mathrm{PDF}}$, parameterized by an $n$-dimensional vector $\bm{\eta} \in \mathbb{R}^n$:
\begin{equation}
p(q \mid \rho) = f_{\mathrm{PDF}}\!\left(q \mid \rho; \bm{\eta}\right).
\label{eq:problem-def-dist}
\end{equation}

The admissible distributions are required to satisfy fundamental physical properties of traffic flow. Under the zero-density condition ($\rho = 0$), the absence of vehicles implies that both traffic flow and its variability vanish. Under the jam-density condition ($\rho = \rho_{\mathrm{jam}}$), vehicles are
fully congested, resulting in zero flow with negligible variability. For intermediate density, the traffic flow value must be positive, and its variability must also be positive.

These physical requirements impose the following moment-based constraints on
the conditional distribution $p(q \mid \rho)$. Let $\mathbb{E}[\cdot]$ and
$\mathbb{S}[\cdot]$ denote the expectation and standard deviation, respectively:
\begin{equation}
\label{eq:boundary_constraints}
\begin{cases}
\mathbb{E}_{q \sim p(q \mid \rho)}[q] =
\mathbb{S}_{q \sim p(q \mid \rho)}[q] = 0,
& \rho \in \{0, \rho_{\mathrm{jam}}\}, \\
0<\mathbb{E}_{q \sim p(q \mid \rho)}[q] <\infty,
& \rho \in (0, \rho_{\mathrm{jam}}), \\
0<\mathbb{S}_{q \sim p(q \mid \rho)}[q] < \infty,
& \rho \in (0, \rho_{\mathrm{jam}}).
\end{cases}
\end{equation}
The modeling problem addressed in this study is therefore to construct a
flexible yet physically consistent parameterization of $p(q \mid \rho)$ that
satisfies these constraints while accurately capturing the stochastic
variability observed in empirical traffic data.

\subsection{Semiparametric Framework for SFD Models}

A common data-driven approach to modeling stochastic traffic flow relationships is to learn distributional parameters as functions of traffic density using nonparametric models such as neural networks \citep{bishop1994mixture, theis2015generative}. 
While such approaches offer substantial modeling flexibility and expressive capability, they do not inherently guarantee physical consistency with respect to the fundamental properties of traffic flow.
In particular, without explicitly enforcing the boundary and non-negativity conditions specified in Eq.~\eqref{eq:boundary_constraints}, unconstrained nonparametric models may produce physically implausible behavior, especially in regimes near zero density or jam density.

To address this limitation, this study adopts a semiparametric modeling framework in which physical constraints are explicitly embedded into the parameterization of the conditional distribution $p(q \mid \rho)$. The central idea begins with separating the distributional parameters into constrained and unconstrained components.

\begin{definition}[Partition of Distribution Parameters]
Let $\bm{\eta} \in \mathbb{R}^n$ ($n \ge 2$) denote the parameter vector governing the conditional distribution $p(q \mid \rho)$. The parameter vector $\bm{\eta}$ is defined by a mutually exclusive partition:
\begin{equation}
\bm{\eta}^\mathsf{T} = 
\begin{bmatrix} 
\bm{\eta}_{\mathrm{con}}^\mathsf{T} & \bm{\eta}_{\mathrm{free}}^\mathsf{T} 
\end{bmatrix},
\label{eq:partition_def}
\end{equation}
where:
\begin{itemize}
    \item $\bm{\eta}_{\mathrm{con}} \in \mathbb{R}^2$ represents the \textbf{constrained} component determined by construction through physical constraints on the first two moments (the conditional expectation and standard deviation).
    \item $\bm{\eta}_{\mathrm{free}} \in \mathbb{R}^{n-2}$ represents the \textbf{free} component that is unconstrained and capture higher-order distributional characteristics.
\end{itemize}
\end{definition}


Under this framework, the distributional parameters are derived by integrating data-driven learning with strict physical constraints. The unconstrained parameters $\bm{\eta}_{\mathrm{free}}$ are learned
directly from data using a neural network, leveraging its universal approximation capability \citep{hornik1989multilayer}. 
In contrast, the constrained parameters $\bm{\eta}_{\mathrm{con}}$ are obtained analytically by enforcing consistency between the theoretical moments of the distribution and the macroscopic traffic characteristics.

Let $\mu(\bm\eta)$ and $\sigma(\bm\eta)$ denote the theoretical expectation and standard deviation derived from the probability density function  $f_\mathrm{PDF}(q\mid\rho;\ \bm\eta)$. 
Meanwhile, we introduce two \textbf{semiparametric functions} $m(\rho)$ and $s(\rho)$, representing the conditional expectation and standard deviation of traffic flow as functions of traffic density. 
These functions, whose detailed structure will be discussed in Section \ref{sec:semiparametric-models}, satisfy the physical constraints defined in Eq.~\eqref{eq:boundary_constraints} by construction.
The consistency between the distributional moments and macroscopic traffic flow characteristics is described by a moment-matching system as follows: 
\begin{equation}\label{eq:moment_matching}
\begin{aligned}
    \mathbb{E}_{q\sim f_\mathrm{PDF}}[q] &= \mu\left(\begin{bmatrix}
    \bm\eta_\mathrm{con}^\mathsf T &\bm\eta_\mathrm{free}^\mathsf T
\end{bmatrix}^\mathsf T\right) = m(\rho; \bm\varphi, \bm h), \\
\mathbb{S}_{q\sim f_\mathrm{PDF}}[q]&=
\sigma\left(\begin{bmatrix}
    \bm\eta_\mathrm{con}^\mathsf T &\bm\eta_\mathrm{free}^\mathsf T
\end{bmatrix}^\mathsf T\right) = s(\rho; \bm\varphi, \bm h),
\end{aligned}
\end{equation}
where the semiparametric functions consist of a parametric component governed by domain-specific parameters $\bm{\varphi}$ (e.g., jam density) and a nonparametric correction term
$\bm{h}$ that captures the deviations from the underling traffic flow relationship.

The correction term $\bm{h}$ and the unconstrained distribution parameters $\bm{\eta}_{\mathrm{free}}$ are generally unknown and must be inferred from data. To provide a flexible representation of their dependence on traffic density, we introduce a neural network mapping that generates these values directly from the density input.
That is, a single neural network model $\mathrm{DNN}(\cdot)$, parameterized by $\bm{\theta}$, is employed to jointly model the latent correction term $\bm{h}$ and the unconstrained parameters $\bm{\eta}_{\mathrm{free}}$:
\begin{equation}\label{eq:dnn}
\begin{bmatrix}
\bm{h}^\mathsf{T} & \bm{\eta}_{\mathrm{free}}^\mathsf{T}
\end{bmatrix}^\mathsf{T}
=
\mathrm{DNN}(\rho; \bm{\theta}).
\end{equation}
The constrained parameters $\bm{\eta}_{\mathrm{con}}$ are subsequently
determined by solving the moment-matching system in
Eq.~\eqref{eq:moment_matching}. For the proposed framework to be well-defined, it is necessary that this system admits a unique solution.
To guarantee this property, we establish the following proposition for
location--scale families of distributions.

\begin{proposition}[Unique Solvability for Moment Matching]\label{prop-uni}
\label{prop:moment_matching}
Suppose the conditional distribution $f_{\mathrm{PDF}}(q \mid \rho; \bm{\eta})$ belongs to a location-scale family with well-defined first and second moments. Let $\bm{\eta}_{\mathrm{con}}$ represent the location and scale parameters. Then, for any intermediate density $\rho \in (0,\rho_{\mathrm{jam}})$ and any admissible unconstrained parameter vector $\bm{\eta}_{\mathrm{free}}$, the moment-matching system defined in Eq.~\eqref{eq:moment_matching} admits a unique solution for $\bm{\eta}_{\mathrm{con}}$.
\end{proposition}
\begin{proof}
Let $\bm\eta_{\mathrm{con}} = [\eta_1, \eta_2]^\mathsf{T}$, where $\eta_1$ is the location parameter and $\eta_2 > 0$ is the scale parameter. By the definition of a location-scale family, a random variable $X \sim f_{\mathrm{PDF}}(q\mid\rho; \bm\eta)$ can be expressed as an affine transformation of a base random variable $Z$:
$$
X = \eta_1 + \eta_2Z,
$$
where $Z$ follows the standardized distribution with location $\eta_1=0$ and scale $\eta_2=1$. 
Crucially, since the shape of $Z$ is governed solely by the unconstrained parameters $\bm{\eta}_{\mathrm{free}}$, we can express its expectation and variance as explicit functions of $\bm{\eta}_{\mathrm{free}}$, denoted by $\mu_0(\cdot)$ and $\sigma_0^2(\cdot)$ respectively:
$$
\mathbb{E}[Z] = \mu_0(\bm\eta_{\mathrm{free}}), \quad \mathrm{Var}[Z] = \sigma_0^2(\bm\eta_{\mathrm{free}}).
$$
By the linearity of expectation and properties of variance, the moments of $X$ are given by:
$$
\mathbb{E}[X] = \eta_1 + \eta_2\mu_0(\bm\eta_{\mathrm{free}}),\quad \mathrm{Var}[X] =\eta_2^2\sigma_0^2(\bm\eta_{\mathrm{free}}).
$$
For any given density $\rho \in (0, \rho_{\mathrm{jam}})$, the corresponding flow expectation $m(\rho)$ and standard deviation $s(\rho)$ are admissible provided they satisfy the boundary constraints in Eq.~\eqref{eq:boundary_constraints}, i.e., 
$$
0 <m(\rho)< \infty, \quad 0<s(\rho) < \infty.
$$
To enforce consistency, the moment-matching system is specified as follows:
$$
\eta_1 + \eta_2\mu_0(\bm\eta_{\mathrm{free}}) = m(\rho),\quad \eta_2^2\sigma_0^2(\bm\eta_{\mathrm{free}}) = s^2(\rho).
$$
Since $\eta_2 > 0$ and $s(\rho) > 0$, taking the positive square root of the variance equation yields the unique solutions for the location and scale parameters:$$\eta_1= m(\rho)-s(\rho)\frac{\mu_0(\bm\eta_{\mathrm{free}}) }{\sigma_0(\bm\eta_{\mathrm{free}})}, \quad \eta_2 = \frac{s(\rho)}{\sigma_0(\bm\eta_{\mathrm{free}})}.$$
\end{proof}


The above proposition guarantees that the moment-matching system
admits a unique solution for location-scale families of distributions.
Consequently, once the semiparametric functions $m(\rho;\bm \varphi,\bm h)$ and $s(\rho;\bm \varphi,\bm h)$ and the unconstrained parameters $\bm{\eta}_{\mathrm{free}}$ are specified, the constrained parameters $\bm{\eta}_{\mathrm{con}}$ are uniquely determined.
Because the latent correction term $\bm{h}$ and the unconstrained
parameters $\bm{\eta}_{\mathrm{free}}$ are generated by the neural
network defined in Eq.~\eqref{eq:dnn}, the complete parameter vector
$\bm{\eta}$ becomes a deterministic function of density $\rho$,
parameterized by the neural network weights $\bm{\theta}$ and the
physical parameters $\bm{\varphi}$. We denote this mapping as
\begin{equation}
\bm{\eta}^\mathsf T \left(\rho;\bm{\theta},\bm{\varphi}\right) 
= \begin{bmatrix}
    \bm{\eta}_\mathrm{con}^\mathsf T\left(\rho;\bm{\theta},\bm{\varphi}\right) & \bm{\eta}_\mathrm{free}^\mathsf T\left(\rho;\bm{\theta}\right)
\end{bmatrix}
\end{equation}
The resulting modeling procedure can therefore be summarized as the
following computational chain:
\begin{equation}\label{eq:comp_chain}
\rho
\xrightarrow[\text{Eq.~\eqref{eq:dnn}}]{\bm{\theta}}
(\bm{h}, \bm{\eta}_{\mathrm{free}})
\xrightarrow[\text{Eq.~\eqref{eq:moment_matching}}]{\bm{\varphi}}
(\bm{\eta}_{\mathrm{con}}, \bm{\eta}_{\mathrm{free}})
\equiv
\bm{\eta}
\longrightarrow
f_{\mathrm{PDF}}(q \mid \rho; \bm{\eta}).
\end{equation}
Finally, model parameters are estimated via maximum likelihood using
the observed dataset $\{(\rho_i, q_i)\}_{i=1}^N$:
\begin{equation}
\hat{\bm{\theta}}, \hat{\bm{\varphi}}
=
\arg\max_{\bm{\theta}, \bm{\varphi}}
\sum_{i=1}^N
\log f_{\mathrm{PDF}}\!\left(q_i \mid \rho_i;
\bm{\eta}(\rho_i; \bm{\theta}, \bm{\varphi})\right).
\end{equation}
The semiparametric framework is therefore well defined for
location-scale families of distributions, ensuring that the resulting
conditional distribution satisfies the imposed physical constraints by
model construction. 
Other distribution families may also be incorporated into the framework, although additional modifications may be required. The extension of the framework will be further discussed in Section~\ref{sec:discussion}. The overall modeling procedure is intuitively summarized in Fig.~\ref{fig:flowchart}. 

\begin{figure}
     \centering    \includegraphics[width=\linewidth]{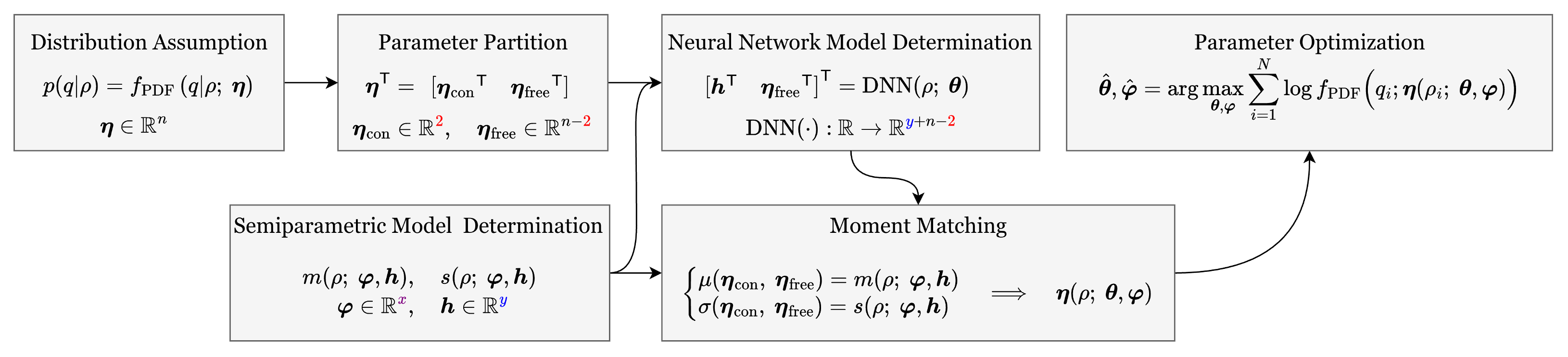}
     \caption{Flowchart for Stochastic FD modeling under Semiparametric Framework}
    \label{fig:flowchart}
\end{figure}

\subsection{Structure Specification of Semiparametric Functions} \label{sec:semiparametric-models}
This section specifies candidate functional forms for the expectation $m(\rho)$ and standard deviation $s(\rho)$ of
traffic flow $q$ given density $\rho$, associated with the conditional distribution $p(q \mid \rho)$. 
Under the physical constraints in Eq.~\eqref{eq:boundary_constraints}, both $m(\rho)$ and $s(\rho)$ belong to the same function class that satisfies specific boundary conditions.

\begin{definition}[Class of Physically Admissible Functions]
Let $\mathcal{F}_{\mathrm{phy}}$ denote the class of functions
\begin{equation}
\mathcal{F}_{\mathrm{phy}}
=
\bigl\{
f \in C([0,\rho_{\mathrm{jam}}]) \mid
f(\rho)\ge 0,\;
f(0)=0,\;
f(\rho_{\mathrm{jam}})=0
\bigr\},
\end{equation}
where $C([0,\rho_{\mathrm{jam}}])$ denotes the space of continuous
real-valued functions defined on the interval $[0,\rho_{\mathrm{jam}}]$.
\end{definition}

Among functions in $\mathcal{F}_{\mathrm{phy}}$, the quadratic
flow-density relationship implied by the classical Greenshields model represents the simplest parametric example.
Building on this baseline, we introduce two semiparametric extensions that preserve physical admissibility while enabling data-driven flexibility through NN-based correction terms.

\subsubsection{Quadratic Function with Neural Correction (QwNC)}
The classical Greenshields model assumes a linear relationship between traffic density $\rho$ and space-mean speed $v$. Since traffic flow is defined as $q=\rho v$, this assumption implies a quadratic relationship between density and flow. The Greenshields formulation is given by
\begin{equation}
v_{\mathrm{G}}(\rho)=v_{\mathrm{free}}\left(1-\frac{\rho}{\rho_{\mathrm{jam}}}\right), \qquad
q_{\mathrm{G}}(\rho)=\rho(\rho_{\mathrm{jam}}-\rho)\frac{v_{\mathrm{free}}}{\rho_{\mathrm{jam}}},
\end{equation}
where $v_{\mathrm{free}}$ denotes the free-flow speed and $\rho_{\mathrm{jam}}$ the jam density. The resulting quadratic function attains zero flow at both $\rho=0$ and $\rho=\rho_{\mathrm{jam}}$, and therefore belongs to $\mathcal{F}_{\mathrm{phy}}$.

To enhance modeling flexibility while preserving physical consistency, the constant coefficient $v_{\mathrm{free}}/\rho_{\mathrm{jam}}$ is replaced by a density-dependent correction term $c(\rho)$. This term corresponds to one element of the latent vector $\bm h$ generated by the neural network defined in Eq.~\eqref{eq:dnn}. The resulting Quadratic Function with Neural Correction is expressed as
\begin{equation}
f_{\mathrm{QwNC}}(\rho)=\rho(\rho_{\mathrm{jam}}-\rho)c(\rho).
\end{equation}
To ensure numerical stability and enforce physical non-negativity, all positive quantities are mapped through the Softplus function $\sigma(x)=\ln(1+e^{x})$. 
In particular, the jam density is parameterized as $\rho_{\mathrm{jam}}=\sigma(\rho_j)$, where $\rho_j$ is an element of the physical parameter vector $\bm\varphi$. The density-dependent correction term $c(\rho)$ is constrained analogously via $\sigma(c(\rho))$, yielding the final formulation
\begin{equation}
f_{\mathrm{QwNC}}(\rho)
=
\rho\bigl(\sigma(\rho_j)-\rho\bigr)\sigma\bigl(c(\rho)\bigr).
\end{equation}

Although the QwNC formulation is motivated by extending the mean flow
relationship implied by the Greenshields model, it satisfies the physical constraints in Eq.~\eqref{eq:boundary_constraints} by construction. Accordingly, the same functional form can be leveraged to model both the conditional expectation and the conditional standard deviation of traffic flow.


\subsubsection{Beta Function with Neural Correction (BwNC)}
To further enhance structural flexibility beyond the quadratic form, we introduce a functional specification inspired by the shape characteristics of the Beta distribution.
Specifically, two shape parameters, $a$ and $b$, are introduced to transform the quadratic term $\rho(\sigma(\rho_j)-\rho)$ into the more flexible form
$\rho^{a}\bigl(\sigma(\rho_j)-\rho\bigr)^{b}$.
The parameters $a$ and $b$ regulate the curvature near low and high density regimes, respectively, enabling the model to capture skewed or asymmetric patterns observed in empirical data.

To ensure physical admissibility and numerical stability, the Softplus function is applied to the exponents, enforcing positivity. In addition, the density difference term, $(\sigma(\rho_j)-\rho)$, is rectified using a $\max(0,\cdot)$ operator to prevent negative values, thereby guaranteeing that the power function remains well-defined across the entire density domain.
The resulting Beta Function with Neural Correction is given by
\begin{equation}
f_{\mathrm{BwNC}}(\rho)
=\rho^{\sigma(a)}
\max\bigl(0,\sigma(\rho_j)-\rho\bigr)^{\sigma(b)}
\sigma\bigl(c(\rho)\bigr).
\end{equation}
This formulation preserves the physical boundary conditions and provides additional flexibility to capture asymmetric relationships. 
The BwNC formulation can be used to model both the expectation and the standard deviation of conditional traffic flow distribution.

%
%
\section{Specification of SFD Models}
Given the semiparametric framework introduced in the preceding section,
a SFD model is fully specified once the conditional distribution $p(q \mid \rho)$ and the semiparametric forms of the conditional expectation $m(\rho)$ and standard deviation $s(\rho)$ are defined. 
In this section, the framework is instantiated by assuming that the conditional distribution follows a Skew-Normal distribution. For completeness, the Gaussian distribution is discussed as a special case obtained by setting skewness parameter as zero.

\subsection{Skew-Normal Distribution}
The stochastic fundamental diagram model is instantiated by assuming the conditional distribution $p(q \mid\rho)$ to follow the Skew-Normal distribution with parameter as $\bm\eta=\{\xi, \omega, \alpha\}$: 
\begin{equation}
\begin{aligned}
     p(q\mid\rho) &= f_{\mathrm{SkewNormal}}\left(q\mid\rho; \
         \xi, \omega, \alpha \right) \\
     &=\sqrt{\frac{2}{\pi}}\frac{1}{\omega}\exp\left(-\frac{1}{2} \left(\frac{q-\xi}{\omega}\right)^2 \right) \ \Phi\left(\alpha \cdot \frac{q-\xi}{\omega}\right) ,
\end{aligned}
\end{equation}
where $\Phi(\cdot)$ denote the standard normal cumulative distribution function.  
Following the Proposition~\ref{prop-uni}, location parameter $\xi$ and scale parameter $\omega$ are regarded as constrained, i.e., $\bm\eta_\mathrm{con}=[\xi,\ \omega]^\mathsf T$. While the skewness parameter $\alpha$ forms the unconstrained parameter $\bm\eta_\mathrm{free}=[\alpha]^\mathsf T$, which is learned directly from data.

\begin{table}[!tbp]
\centering
\caption{Summary of Functional Forms and Parameters for Stochastic FD Models under Skew-Normal Distribution}\label{tab:sn-model-sum}
\begin{tabular}{l|cc}
\toprule
 & \textbf{SN-QwNC} & \textbf{SN-BwNC} \\
\midrule
 $m(\rho;\ \bm\varphi, \bm h)$
&
$\rho   \bigl(\sigma(\rho_{j})-\rho\bigr)   \sigma\bigl(c_1 \bigr)$
&
$\rho^{\sigma(a_1)} 
\max\bigl(0, \ \sigma(\rho_{j})-\rho\bigr)^{\sigma(b_1)} 
\sigma\bigl(c_1\bigr)$ \\
 $s(\rho;\ \bm\varphi, \bm h)$
&
$\rho   \bigl(\sigma(\rho_{j})-\rho\bigr)   \sigma\bigl(c_2\bigr)$
&
$\rho^{\sigma(a_2)} 
\max\bigl(0, \ \sigma(\rho_{j})-\rho\bigr)^{\sigma(b_2)} 
\sigma\bigl(c_2\bigr)$\\

$\bm \varphi$ & $\rho_j$ & $\rho_j, a_1, b_1, a_2, b_2$\\

$\bm h$ & $c_1, c_2$ & $c_1, c_2$\\
DNN mapping & \multicolumn{2}{c}{$\rho\rightarrow [c_1\quad c_2 \quad \alpha] ^\mathsf T $}\\
$\mathrm{DNN}(\rho; \bm\theta)$ & \multicolumn{2}{c}{
$\bm W_3 \ \mathrm{silu}(\bm W_2 \ \mathrm{silu}(\bm W_1 \rho + \bm b_1) + \bm b_2) + \bm b_3$
}\\
$\bm\theta$ & \multicolumn{2}{c}{$\bm W_1, \bm W_2, \bm W_3, \bm b_1, \bm b_2, \bm b_3$}\\
\bottomrule
\end{tabular}
\end{table}

Two model variants based on Skew-Normal distribution, namely SN-QwNC and SN-BwNC, are established based on the functional forms defined in Section \ref{sec:semiparametric-models}. The SN-QwNC model applies the $f_\mathrm{QwNC}$ for functions of the expectation $m(\rho;\bm\varphi,\bm h)$ and standard deviation $s(\rho;\bm\varphi,\bm h)$, whereas the SN-BwNC model utilizes the  $f_\mathrm{BwNC}$. A summary of the functional forms and parameter sets is provided in Table \ref{tab:sn-model-sum}.
For both the SN-QwNC and SN-BwNC specifications, the vector $\bm h$ contains two correction components, one associated with the conditional expectation and one with the conditional standard deviation. The domain-specific parameter vector
$\bm\varphi$ differs across the two variants: in the SN-QwNC model, $\bm\varphi$ contains only density-related parameters (e.g., jam density), whereas in the SN-BwNC model it additionally includes parameters controlling the curvature of
the flow--density relationship. 

The correction terms $\bm h$ and the skewness parameter $\alpha$ are learned jointly using a neural network. Specifically, a three-layer multilayer perceptron (MLP) with Sigmoid Linear Unit (SiLU)
activation functions \citep{hendrycks2016gaussian} is employed.



The constrained parameters $\{\xi,\omega\}$ are determined by matching the theoretical moments of the Skew-Normal distribution to the semiparametric functions $m(\rho;\bm\varphi,\bm h)$ and $s(\rho;\bm\varphi,\bm h)$. This yields the system of equations as follows:
\begin{equation}
\begin{aligned}
\xi + \omega\left(\frac{\alpha^2}{1+\alpha^2}\frac{2}{\pi}\right)^{1/2}
&= m(\rho;\bm\varphi,\bm h),\\
\omega^2\left(1-\frac{\alpha^2}{1+\alpha^2}\frac{2}{\pi}\right)
&= s^2(\rho;\bm\varphi,\bm h),
\end{aligned}
\end{equation}
from which $\xi$ and $\omega$ can be expressed explicitly as functions of $\alpha$, $m(\rho;\bm\varphi,\bm h)$, and $s(\rho;\bm\varphi,\bm h)$, i.e.,
\begin{equation}\label{eq:mmsk}
\begin{aligned}
\omega &=
s(\rho;\bm\varphi,\bm h)
\left(1-\frac{\alpha^2}{1+\alpha^2}\frac{2}{\pi}\right)^{-1/2},\\
\xi &=
m(\rho;\bm\varphi,\bm h)
-\omega\left(\frac{\alpha^2}{1+\alpha^2}\frac{2}{\pi}\right)^{1/2}.
\end{aligned}
\end{equation}

Consequently, all distributional parameters $\{\xi,\omega,\alpha\}$ become
deterministic functions of the density $\rho$, parameterized by the neural network weights $\bm\theta$ and the domain-specific parameters $\bm\varphi$. Model parameters are learned by minimizing a composite loss function consisting of the negative log-likelihood and a jam-density regularization term, i.e.,
\begin{equation}
    \ell = \ell_{\mathrm{NLL}} + \lambda\cdot\ell_{\mathrm{Reg}}.
\end{equation}
Given a dataset $\mathcal{D}=\{(\rho_i,q_i)\}_{i=1}^N$ of observed density-flow pairs, the negative log-likelihood under the Skew-Normal model is
\begin{equation}\label{eq:nll}
\begin{aligned}
\ell_{\mathrm{NLL}}
&= -\log\!\left(\prod_{i=1}^{N} p(q_i|\rho_i)\right) \\
&= N\log\omega_i
+ \frac{1}{2}\sum_{i=1}^{N}\left(\frac{q_i-\xi_i}{\omega_i}\right)^2
- \sum_{i=1}^{N}\log \Phi\!\left(\alpha_i \frac{q_i-\xi_i}{\omega_i}\right),
\end{aligned}
\end{equation}
where
\[
\bm h_i,\, \alpha_i = \mathrm{DNN}(\rho_i; \bm\theta),
\]
\[
\omega_i = s(\rho_i;\bm\varphi,\bm h_i)
\left(1-\frac{\alpha_i^2}{1+\alpha_i^2}\frac{2}{\pi}\right)^{-1/2},
\]
\[
\xi_i =
m(\rho_i;\bm\varphi,\bm h_i)
-\omega_i
\left(\frac{\alpha_i^2}{1+\alpha_i^2}\frac{2}{\pi}\right)^{1/2}.
\]
A regularization term is introduced to stabilize the optimization process and to enforce physical consistency of the estimated jam density:
\begin{equation}\label{eq:reg}
\ell_{\mathrm{Reg}} =
\sum_{i=1}^N \max\bigl(\rho_i - \sigma(\rho_j),\, 0\bigr),
\end{equation}
where $\sigma(\rho_j)$ denotes the jam density reparameterized by Softplus function. The rationale of this regularization is to encourage the learned jam density to be no smaller than any observed traffic density in the dataset. Whenever the estimated jam density falls below observed densities, a penalty is incurred to prevent physically implausible solutions. No penalty is applied when the estimated jam density exceeds all observed densities.

A relatively large weighting factor $\lambda$, e.g., $\lambda=10^2$, is employed to ensure effective enforcement of the regularization during the training phase. In practice, the regularization rapidly guides the jam density estimate toward a physically admissible range while maintaining stable convergence of the likelihood-based optimization.

\subsection{Normal Distribution}
The Normal distribution can be considered as a degenerate case of the Skew-Normal distribution obtained by setting the skewness parameter $\alpha$ to zero. Under this assumption, the conditional distribution of traffic flow given density becomes symmetric, and only the location and scale parameters, $\xi$ and $\omega$, are required. Consistent with the semiparametric framework, these parameters are
treated as constrained parameters $\bm\eta_{\mathrm{con}}$, while no distributional parameter is left as free, i.e.,
$\bm\eta_{\mathrm{free}}=\emptyset$.
The conditional probability density function is given by
\begin{equation}
p(q\mid\rho)
=\frac{1}{\sqrt{2\pi}\,\omega}
\exp\!\left(
-\frac{1}{2}\left(\frac{q-\xi}{\omega}\right)^2
\right).
\end{equation}
Analogous to the Skew-Normal specification, two Normal distribution based SFD models are defined: N-QwNC and N-BwNC. These models employ the same semiparametric functional forms for the conditional expectation and standard deviation as their Skew-Normal counterparts. The only structural difference lies in the
neural network output, which no longer includes a skewness parameter. A summary of the functional forms and associated parameters is provided in Table~\ref{tab:n-model-sum}.

\begin{table}[!tb]
\centering
\caption{Summary of Functional Forms and Parameters for Stochastic FD Models under Normal Distribution}\label{tab:n-model-sum}
\begin{tabular}{l|cc}
\toprule
 & \textbf{N-QwNC} & \textbf{N-BwNC} \\
\midrule
 $m(\rho;\ \bm\varphi, \bm h)$
&
$\rho  \bigl(\sigma(\rho_{j})-\rho\bigr)  \sigma\bigl(c_1 \bigr)$
&
$\rho^{\sigma(a_1)} 
\max\bigl(0, \ \sigma(\rho_{j})-\rho\bigr)^{\sigma(b_1)} 
\sigma\bigl(c_1\bigr)$ \\
 $s(\rho;\ \bm\varphi, \bm h)$
&
$\rho   \bigl(\sigma(\rho_{j})-\rho\bigr)   \sigma\bigl(c_2\bigr)$
&
$\rho^{\sigma(a_2)} 
\max\bigl(0, \ \sigma(\rho_{j})-\rho\bigr)^{\sigma(b_2)} 
\sigma\bigl(c_2\bigr)$\\

$\bm \varphi$ & $\rho_j$ & $\rho_j, a_1, b_1, a_2, b_2$\\

$\bm h$ & $c_1, c_2$ & $c_1, c_2$\\
DNN mapping & \multicolumn{2}{c}{$\rho\rightarrow [c_1\quad c_2 ] ^\mathsf T $}\\
$\mathrm{DNN}(\rho; \bm\theta)$ & \multicolumn{2}{c}{
$\bm W_3 \ \mathrm{silu}(\bm W_2 \ \mathrm{silu}(\bm W_1 \rho + \bm b_1) + \bm b_2) + \bm b_3$
}\\
$\bm\theta$ & \multicolumn{2}{c}{$\bm W_1, \bm W_2, \bm W_3, \bm b_1, \bm b_2, \bm b_3$}\\
\bottomrule
\end{tabular}
\end{table}

For the Normal distribution, the constrained parameters are obtained as a special case of Eq~\eqref{eq:mmsk} by setting the skewness parameter $\alpha=0$. Under this condition, the location and scale parameters reduce to
\begin{equation}
\omega = s(\rho;\bm\varphi,\bm h),
\qquad
\xi = m(\rho;\bm\varphi,\bm h).
\end{equation}
Since there is no skewness parameter, the neural network only outputs the correction terms associated with the semiparametric functions.
Given a dataset
$\mathcal{D}=\{(\rho_i,q_i)\}_{i=1}^N$, the loss function is defined as:
\begin{equation}
\begin{aligned}
    \ell &= \ell_{\mathrm{NLL}} + \lambda    \ell_{\mathrm{Reg}}\\
    &= \sum_{i=1}^N \log s(\rho_i;\bm\varphi,\bm h)
    +
    \frac{1}{2}
    \sum_{i=1}^N
    \left(
    \frac{q_i-m(\rho_i;\bm\varphi,\bm h)}
    {s(\rho_i;\bm\varphi,\bm h)}
    \right)^2
    + \lambda \sum_{i=1}^N \max\bigl(\rho_i - \sigma(\rho_j),\, 0\bigr).
\end{aligned}
\end{equation}

\section{Computational Experiments}
\subsection{MAGIC Dataset} 
The proposed models are evaluated using the MAGIC (Multiple Conditions UAV Group-based High-fidelity Comprehensive Vehicle Trajectory) dataset \citep{ma2022magic}, which provides high-resolution vehicle trajectory data
collected by unmanned aerial vehicles (UAVs). The data were recorded over a three-hour period (7{:}40--10{:}40 AM) along a 4~km segment of the Shanghai Inner Ring Road.

The dataset contains detailed information on vehicle type, position, velocity, and acceleration at a sampling frequency of 25~Hz. From the full dataset, a 200~m two-lane road segment was extracted to ensure coverage of a wide range of traffic conditions, spanning free-flow to congested regimes.
Traffic variables were aggregated into $1$ second intervals. The aggregated speed is measured in kilometers per hour (km/h), while traffic density is expressed in vehicles per kilometer per lane (veh/km/lane). Traffic flow is computed as the product of density and speed, with units of vehicles per hour per lane (veh/h/lane).


\subsection{Experiment Setup}
A major challenge in our evaluation is the limited and imbalanced nature of the dataset. Since only three hours of tabular data were collected from a single location, the limited data volume introduces high variance in the evaluation metrics. Furthermore, the skewed distribution toward low-density samples could overemphasize model performance in the free-flow state while ignoring the performance in the congested state.

Stratified $k$-fold cross-validation combined with a sample-weighting scheme (as illustrated in Figure \ref{fig:exp_setup}) is used to address the aforementioned challenge. 
The procedure is detailed as follows:
\begin{itemize}
    \item {Stratified Partitioning}. The entire density range is partitioned into $b=10$ equal bins, with $N_m$ data points in $m$-th bin. The data within each bin is further split into $k=5$ equal folds. A complete fold is constructed by taking one partition from each bin without replacement, resulting in a total of $k$ folds.
    \item {Sample Weighting}.
    For each experimental run, one fold $\mathcal{T}_j$ ($j \in \{1, \dots, k\}$) serves as the test set, while the remaining $k-1$ folds are used for training. 
    Each test sample $i \in \mathcal{T}_j$ is weighted inversely proportional to its bin's total population. Specifically, the weight is 
    \begin{equation}
        w_i = \frac{1}{N_{b(i)}},
    \end{equation}
     where $b(i)$ is the bin index of sample $i$.
\end{itemize}
The weight $w_i$ is explicitly integrated into the calculation of each specific evaluation metric (detailed in Section \ref{sec:met}) to ensure an unbiased assessment. Let $S_j$ denote the aggregated score of a given metric evaluated on the test fold $\mathcal{T}_j$. 
After iterating through all $k$ folds, the model's final performance is summarized by the mean ($\mu_S$) and standard deviation ($\sigma_S$) of the $k$ aggregated scores:
\begin{equation}
    \mu_S = \frac{1}{k} \sum_{j=1}^k S_j, \quad \sigma_S = \sqrt{\frac{1}{k} \sum_{j=1}^k \left(S_j - \mu_S\right)^2}.
\end{equation}

\begin{figure}[!tbp]
    \centering
    \includegraphics[width=1\linewidth]{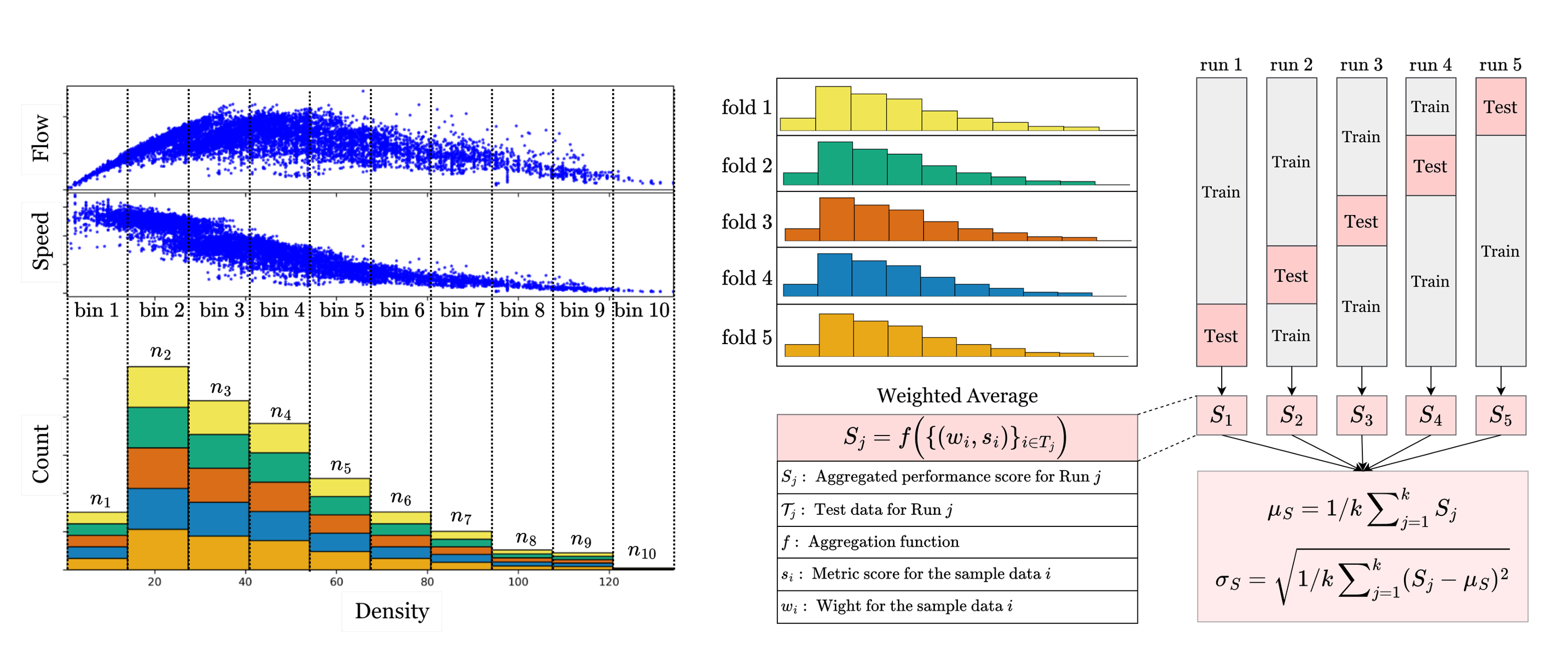}
    \caption{Stratified 5-fold cross-validation combined with a sample-weighting scheme}
    \label{fig:exp_setup}
\end{figure}

\subsection{Metrics}\label{sec:met} 
Five metrics are used for a comprehensive evaluation of the proposed models. The first two are probabilistic metrics, which evaluate the quality of the predicted output distribution. The remaining three are deterministic metrics, which assess the accuracy of the predicted conditional mean.
For notation clarity across all metrics, let $q_i$ denote the ground truth flow observation for sample $i$. Let $f_{\mathrm{PDF}}\left(\cdot \mid \rho_i; \bm{\eta}(\rho_i; \bm{\theta}, \bm{\varphi})\right)$ denote the predicted conditional probability density function given density $\rho_i$, where the distribution parameters $\bm{\eta}$ are produced by a model parameterized by learnable weights $\bm{\theta}$ and $\bm{\varphi}$.

\subsubsection{Probabilistic Metrics}
\paragraph{Weighted Continuous Ranked Probability Score (WCRPS)}
The CRPS is a proper scoring rule commonly used to evaluate the probabilistic forecasts. For a predictive distribution $f$ and an actual observation $y$, it is defined as:
\begin{equation}
    \text{CRPS}(f, y) = \int_{-\infty}^{\infty} \left( \int_{-\infty}^{x} f(t)\mathrm{d}t -\mathbf{1}_{x \geq y}\right)^2 \mathrm{d}x .
\end{equation}
Since integrating the cumulative distribution is often computationally intractable, CRPS can be approximated in a expectation form using Monte Carlo simulation \citep{gneiting2007strictly}:
\begin{equation}
    \text{CRPS}(f,y) = \mathbb{E}_{X\sim f}\bigl[|X-y|\bigl] - \frac{1}{2}\mathbb{E}_{X,X'\sim f}\bigl[|X-X'|\bigl].
\end{equation}
When incorporating the sample weights $w_i$, the aggregated Weighted CRPS for the $j$-th test fold $\mathcal{T}_j$ is calculated as:
\begin{equation}
    S_{j, \mathrm{WCRPS}} = \frac{\sum_{i \in \mathcal{T}_j} w_i \cdot \text{CRPS}\left(f_{\mathrm{PDF}}\!\left(\cdot \mid \rho_i; \bm{\eta}(\rho_i; \bm{\theta}, \bm{\varphi})\right), q_i\right)}{\sum_{i \in \mathcal{T}_j} w_i}.
\end{equation}
\paragraph{Weighted Negative Log-Likelihood (WNLL)}
The Negative Log-Likelihood measures the goodness-of-fit by penalizing low probabilities assigned to the true observations. By incorporating the weight, the WNLL for the test fold is given by:
\begin{equation}
    S_{j, \mathrm{WNLL}} = -\frac{\sum_{i \in \mathcal{T}_j} w_i \cdot \log f_{\mathrm{PDF}}\!\left(q_i \mid \rho_i; \bm{\eta}(\rho_i; \bm{\theta}, \bm{\varphi})\right)}{\sum_{i \in \mathcal{T}_j} w_i}.
\end{equation}

\subsubsection{Deterministic Metrics}
For deterministic evaluation, we collapse the predictive distribution into a single point estimate $\hat{q}_i$ by taking its mathematical expectation:
\begin{equation}
    \hat{q}_i = \mathbb{E}_{X \sim f_{\mathrm{PDF}}\!\left(\cdot \mid \rho_i; \bm{\eta}(\rho_i; \bm{\theta}, \bm{\varphi})\right)}[X].
\end{equation}

\paragraph{Weighted Mean Absolute Error (WMAE)} The WMAE measures the average magnitude of the absolute errors, weighted by the density bin frequency to prevent bias toward free-flow states:
\begin{equation}
    S_{j, \mathrm{WMAE}} = \frac{\sum_{i \in \mathcal{T}_j} w_i |\hat{q}_i - q_i|}{\sum_{i \in \mathcal{T}_j} w_i}.
\end{equation}

\paragraph{Root Weighted Mean Square Error (RWMSE)}
To heavily penalize larger prediction errors while accounting for the skewed dataset, we compute the Root Weighted Mean Square Error as follows:
\begin{equation}
    S_{j, \mathrm{RWMSE}} = \sqrt{\frac{\sum_{i \in \mathcal{T}_j} w_i (\hat{q}_i - q_i)^2}{\sum_{i \in \mathcal{T}_j} w_i}}.
\end{equation}

\paragraph{Weighted Mean Absolute Percentage Error (WMAPE)}
To provide a relative error perspective, we define the WMAPE as the ratio of the weighted absolute errors to the weighted absolute true values:
\begin{equation}
    S_{j, \mathrm{WMAPE}} = \frac{\sum_{i \in \mathcal{T}_j} w_i |\hat{q}_i - q_i|}{\sum_{i \in \mathcal{T}_j} w_i |q_i|}.
\end{equation}

\subsection{Model Setup}
During the training phase, the model (configured with a hidden size of 16) is updated in mini-batches of 128 via the Adam optimizer \citep{kingma2014adam}. We employ a weight decay of $10^{-5}$ for regularization, alongside $\beta$ coefficients of $(0.9, 0.99)$. The learning rate is initially set to $10^{-2}$ and explicitly decayed to $10^{-3}$ halfway through the 200-epoch training process. 
To ensure numerical stability, $\log\Phi(\cdot)$ is computed using the optimized function \texttt{log\_ndtr} available in PyTorch \citep{paszke2019pytorch} and JAX \citep{frostig2019compiling}.

\subsection{Baselines}
This study utilizes three recently proposed stochastic FD models as baselines, which are representative of two prevailing modeling strategies: 1) random parameter modeling and 2) stochastic error/variance modeling. A common principle in both strategies is the modification of deterministic FD models. To obtain accurate deterministic FD models, the weighted least squares method, adopted in \citep{qu2015fundamental}, is applied to mitigate the influence of unbalanced data during the training phase. 

\subsubsection{Random parameter modeling}
The work of \citet{cheng2024analytical} extends the deterministic S3 model \citep{cheng2021s} by introducing stochasticity into its parameters. Since the original formulation is not explicitly named in the original paper, we refer to it as \textit{S3+LN} for clarity. The model is defined as follows. The S3 speed-density and flow-density relations are:
\begin{equation}\label{eq:s3+ln}
\begin{aligned}
    v_{\mathrm{S3}}(\rho) =
    \frac{v_{\mathrm{free}}}
    {\left(1+\left(\frac{\rho}{\rho_{\mathrm{critical}}}\right)^{m}\right)^{\frac{2}{m}}},
    &\quad
    q_{\mathrm{S3}}(\rho) =
    \frac{\rho\,v_{\mathrm{free}}}
    {\left(1+\left(\frac{\rho}{\rho_{\mathrm{critical}}}\right)^{m}\right)^{\frac{2}{m}}} 
\end{aligned}
\end{equation}
In the stochastic formulation, the free-flow speed and critical speed
are treated as random variables following Log-Normal distributions:
\begin{equation}
\begin{aligned}
    \ln(v_{\mathrm{free}})\sim \mathcal{N}(\mu_{f}, \sigma^2_{f}), 
    &\quad
    \ln(v_{\mathrm{critical}}) \sim \mathcal{N}(\mu_{c}, \sigma^2_{c}),
\end{aligned}
\end{equation}
and the parameter follows
\begin{equation}
    \frac{1}{m}\sim\mathcal{N}
    \left(
        \frac{\mu_{f}-\mu_{c}}{2\ln2},
        \frac{\sigma^2_{f}+\sigma^2_{c}}{(2\ln2)^2}
    \right).
\end{equation}
This framework contains five trainable parameters: $\{\mu_{f}, \sigma_{f}, \mu_{c}, \sigma_{c}, \rho_{\mathrm{critical}}\}$. The free-flow speed and critical speed are parameterized by $\{\mu_{f}, \sigma_{f}\}$ and $\{\mu_{c}, \sigma_{c}\}$, respectively. The critical density $\rho_{\mathrm{critical}}$ is estimated using the weighted least squares approach. 

\subsubsection{Stochastic Error/Variance Modeling}
The work of \citet{lei2024unraveling} extends deterministic fundamental
diagram (FD) models by modeling the residual error using a Gaussian
Process (GP). In this formulation, the overall speed $v(\rho)$ is
represented as the sum of a deterministic baseline model
$f_{\mathrm{det}}(\rho)$ and a stochastic residual term
$\epsilon(\rho)$:
\begin{equation}\label{eq:gp}
\begin{aligned}
    v(\rho) &= f_{\mathrm{det}}(\rho) + \epsilon(\rho), \\
    \epsilon(\rho) &\sim \mathcal{GP}\bigl(0, k(\rho,\rho')\bigr), 
\end{aligned}
\end{equation}
where the error term $\epsilon(\rho)$ follows a zero-mean Gaussian Process governed by a Radial Basis Function (RBF) kernel $k(\rho,\rho')$:
\begin{equation}
    k(\rho,\rho') = \sigma^2\exp\left(-\frac{(\rho-\rho')^2}{2l^2}\right).
\end{equation}
This kernel introduces two additional trainable hyperparameters: the signal variance $\sigma^2$ and the length-scale $l$, which control the magnitude and smoothness of the stochastic component, respectively.
The Greenshields and S3 deterministic models are used as baselines for the GP framework, resulted in two variants of models denoted as \textit{GS+GP} and \textit{S3+GP}.

\section{Results} 
\subsection{Metric Performance}
\begin{table}[!tb]\centering
\caption{Performance of all evaluated stochastic FD models}\label{tab:estimation_performance}
{\small
\begin{tabular}{lccccc}
\toprule
\multicolumn{6}{c}{\textbf{Flow-Density Relation}}\\
\midrule
Model & CRPS $(\downarrow)$ & NLL $(\downarrow)$ & MAE $(\downarrow)$ & RWMSE $(\downarrow)$ & WMAPE\% $(\downarrow)$ \\
\midrule
S3 + LN \citep{cheng2024analytical}  &$312.977_{\pm12.062}$ & -- &$448.944_{\pm9.658}$ &$551.809_{\pm10.923}$ &$45.594_{\pm1.175}$\\
S3 + GP \citep{lei2024unraveling}&$165.896_{\pm1.902}$ &$7.120_{\pm0.013}$ & $179.611_{\pm3.073}$ & $270.631_{\pm6.595}$ & $18.238_{\pm0.216}$\\
GS + GP \citep{lei2024unraveling}&$172.785_{\pm2.587}$ &$7.139_{\pm0.013}$&  $204.973_{\pm5.739}$ &$274.391_{\pm8.685}$ & $20.816_{\pm0.632}$\\
\midrule
SN-QwNC &$124.936_{\pm3.629}$ & $6.663_{\pm0.037}$ & $172.636_{\pm5.740}$ & $247.023_{\pm8.740}$ & $17.529_{\pm0.531}$\\ 
N-QwNC  & $127.917_{\pm 2.564}$ & $6.749_{\pm 0.040}$ &$173.303_{\pm4.080}$ & $247.880_{\pm7.560}$ &$17.597_{\pm0.338}$\\
SN-BwNC &$\textbf{122.643}_{\pm3.864}$ & $\textbf{6.633}_{\pm0.053}$ & $\textbf{171.736}_{\pm4.616}$ &$\textbf{246.799}_{\pm8.124}$ &$\textbf{17.438}_{\pm0.413}$\\
N-BwNC  & $124.444_{\pm 4.020}$ & $6.683_{\pm 0.046}$ & $172.594_{\pm5.425}$ & $247.220_{\pm8.621}$ & $17.525_{\pm0.488}$\\

\midrule
\multicolumn{6}{c}{\textbf{Speed-Density Relation}}\\
\midrule
Model & CRPS $(\downarrow)$ & NLL $(\downarrow)$ & MAE $(\downarrow)$ & RWMSE $(\downarrow)$ & WMAPE\% $(\downarrow)$ \\
\midrule
S3 + LN \citep{cheng2024analytical} &$4.566_{\pm0.130}$ & -- &$6.637_{\pm0.142}$ &$7.795_{\pm0.180}$ & $24.756_{\pm0.573}$\\
S3 + GP \citep{lei2024unraveling}&$2.997_{\pm0.049}$ &$3.166_{\pm0.014}$ & $3.700_{\pm0.061}$ & $6.894_{\pm0.161}$ & $13.801_{\pm0.254}$\\
GS + GP \citep{lei2024unraveling}&$3.069_{\pm0.034}$ &$3.185_{\pm0.013}$&  $3.929_{\pm0.072}$ &$5.493_{\pm0.149}$ & $14.654_{\pm0.291}$\\
\midrule
SN-QwNC &$2.590_{\pm0.072}$ & $2.708_{\pm0.039}$ & $3.605_{\pm0.084}$ & $5.320_{\pm0.140}$ & $13.448_{\pm0.338}$\\ 
N-QwNC  & $2.661_{\pm0.047}$ & $2.795_{\pm0.041}$ & $3.634_{\pm0.062}$ & $5.368_{\pm0.110}$ & $13.556_{\pm0.259}$\\
SN-BwNC & $\textbf{2.548}_{\pm0.057}$ & $\textbf{2.679}_{\pm0.054}$ & $\textbf{3.594}_{\pm0.069}$ & $\textbf{5.322}_{\pm0.128}$ & $\textbf{13.407}_{\pm0.279}$\\
N-BwNC  & $2.592_{\pm0.059}$ & $2.728_{\pm0.048}$ & $3.606_{\pm0.068}$ & $5.321_{\pm0.127}$ & $13.450_{\pm0.271}$\\
\bottomrule
\end{tabular}
}
\end{table}

\begin{figure}[!th]
    \centering
    \begin{subfigure}{0.45\textwidth}
        \includegraphics[width=\linewidth]{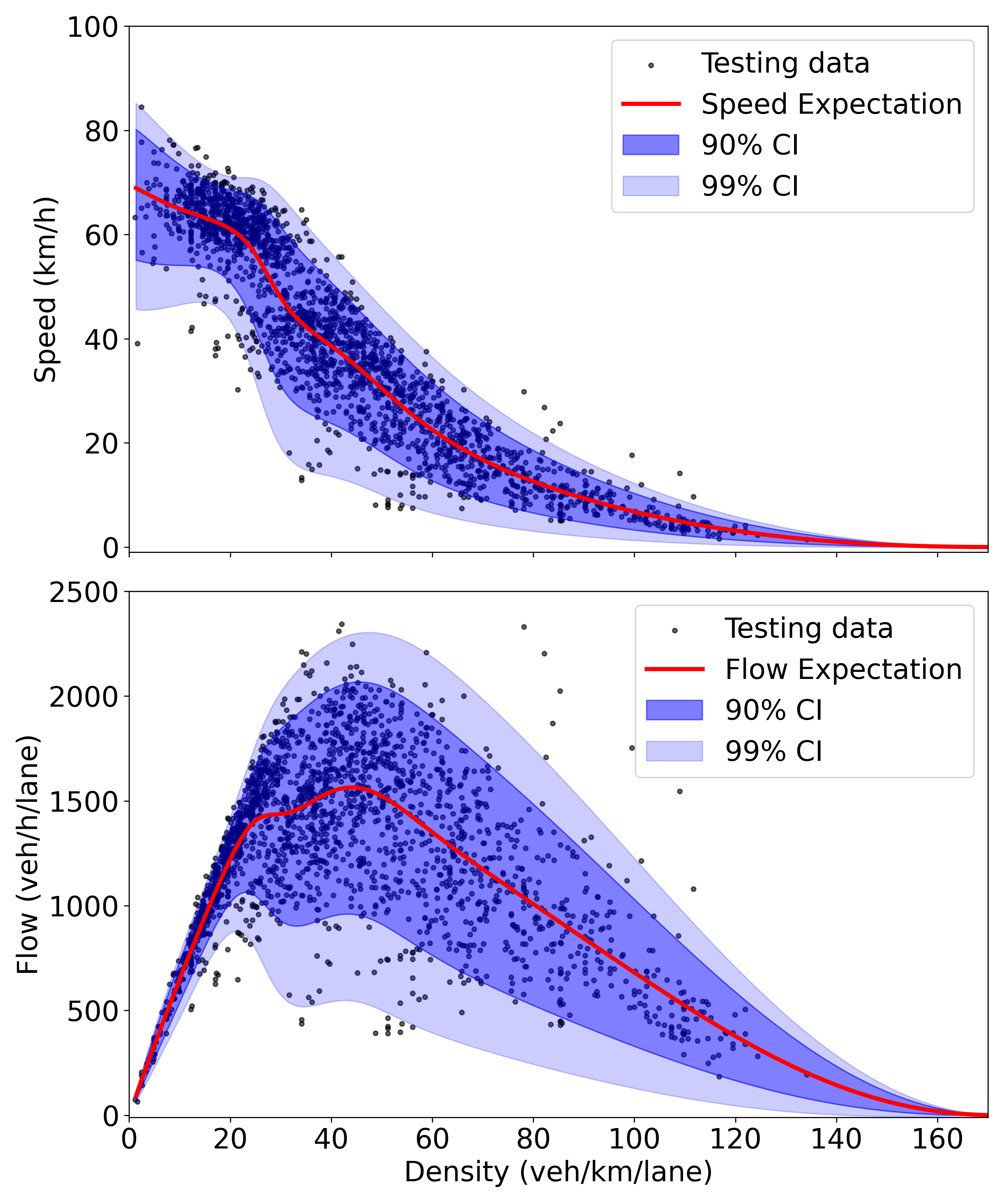}
        \caption{SN-BwNC}
    \end{subfigure}
    \begin{subfigure}{0.45\textwidth}
        \includegraphics[width=\linewidth]{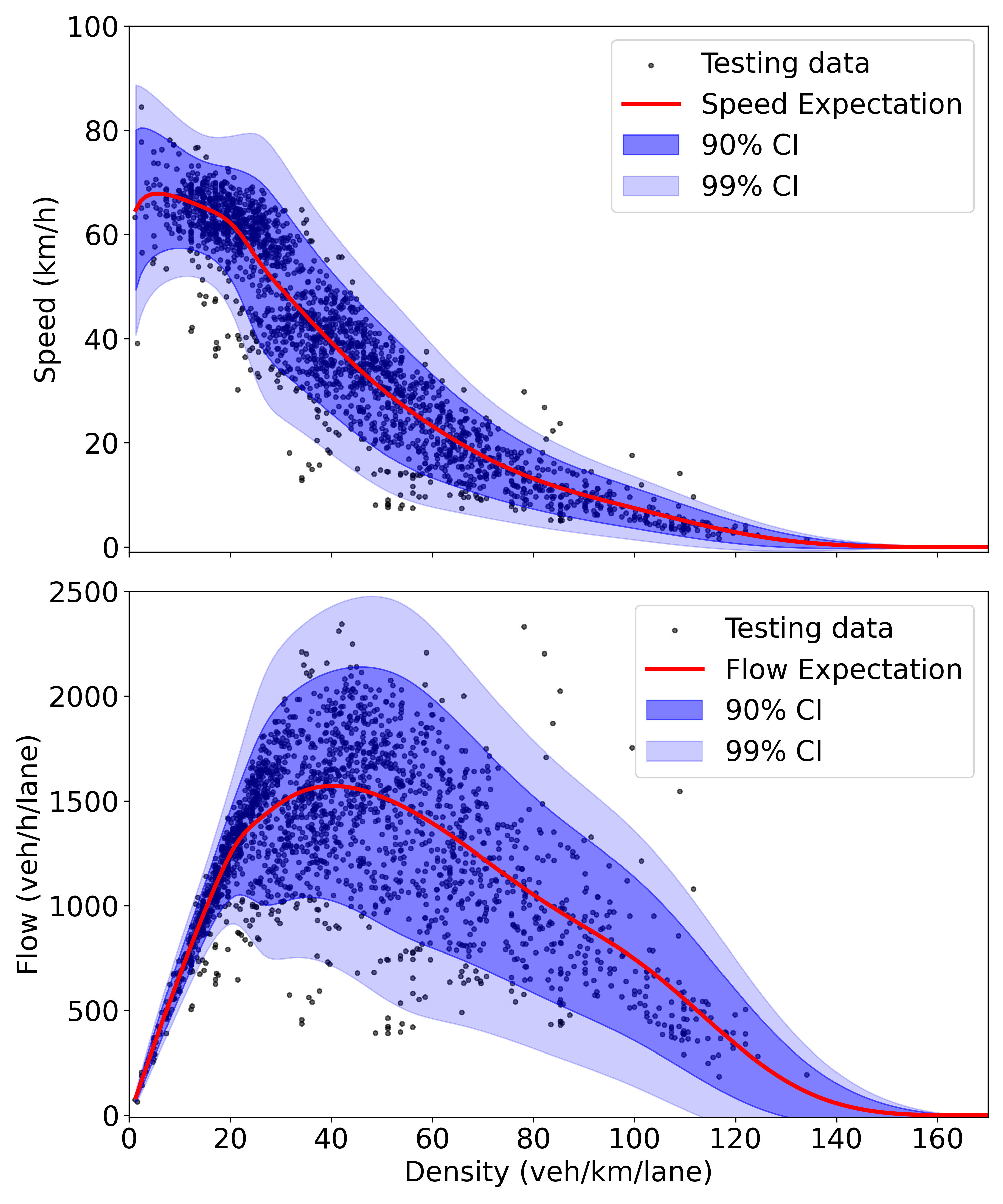}
        \caption{N-BwNC}
    \end{subfigure}
    \vspace{0.5em}

    \begin{subfigure}{0.45\textwidth}
        \includegraphics[width=\linewidth]{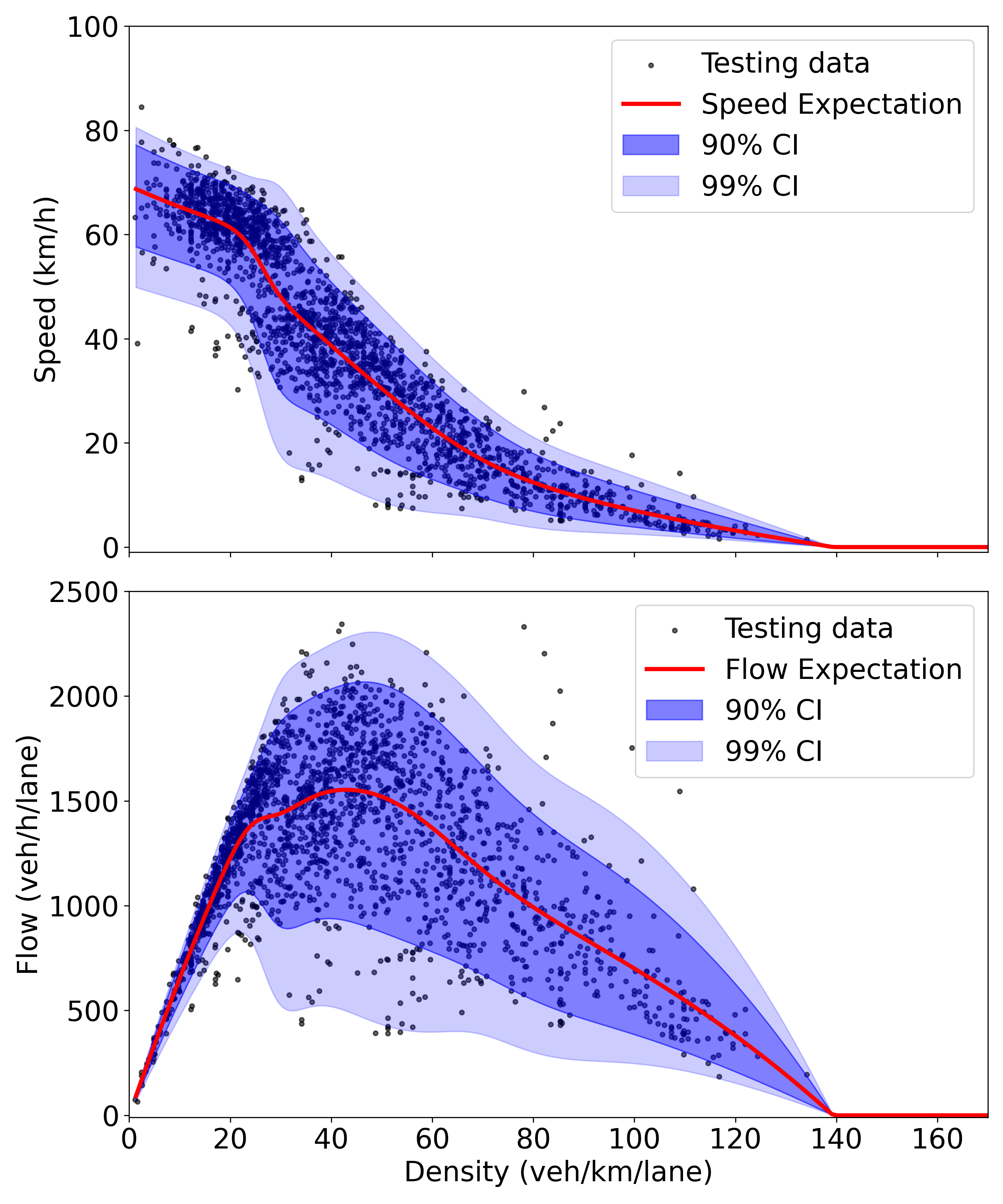}
        \caption{SN-QwNC}
    \end{subfigure}
    \begin{subfigure}{0.45\textwidth}
        \includegraphics[width=\linewidth]{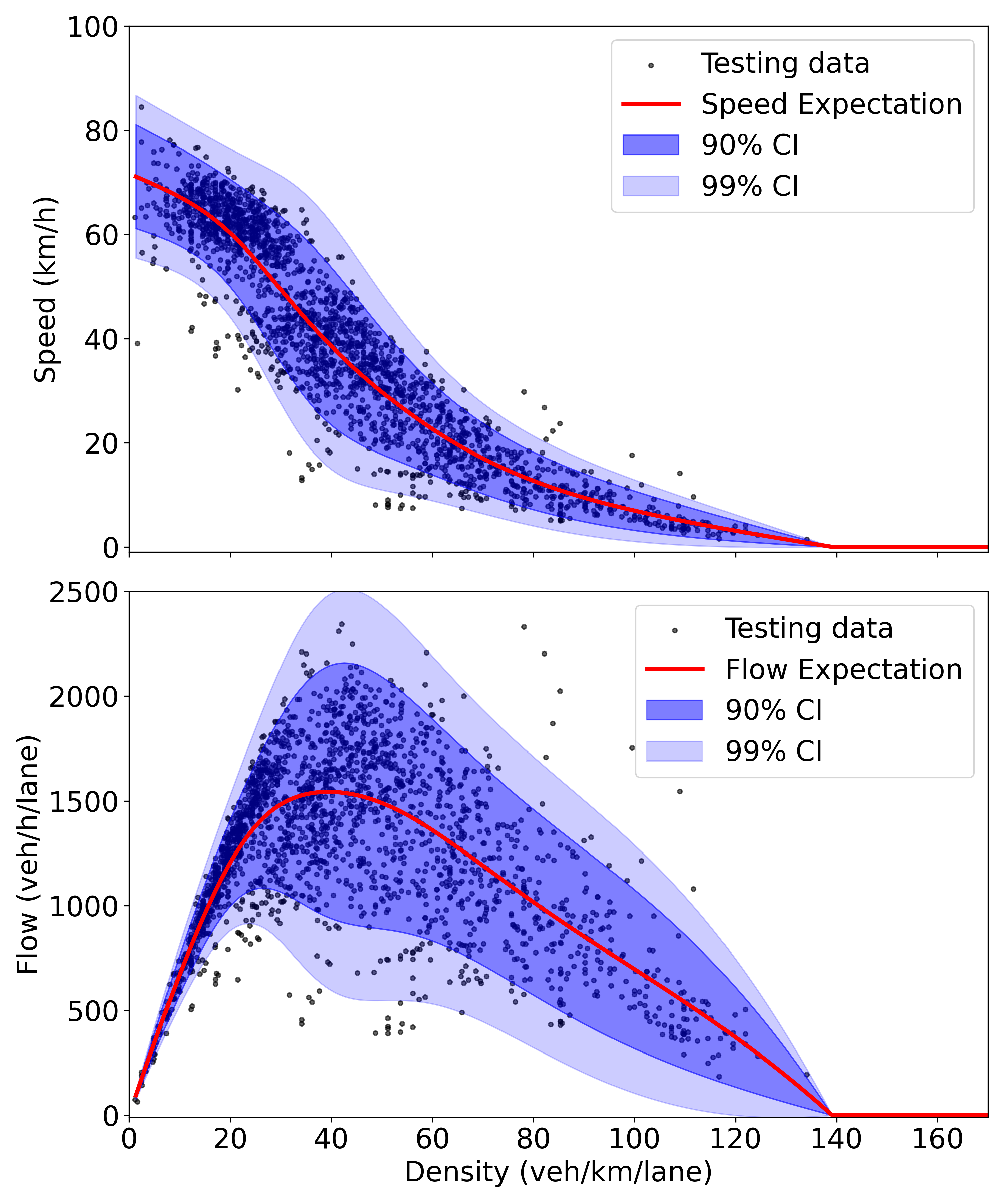}
        \caption{N-QwNC}
    \end{subfigure}
    \caption{Modeling results from Stochastic FD Models under the Semiparametric Framework}
    \label{fig:five_subfigs}
\end{figure}

\begin{figure}[!th]
    \centering
    \begin{subfigure}{0.45\textwidth}
        \includegraphics[width=\linewidth]{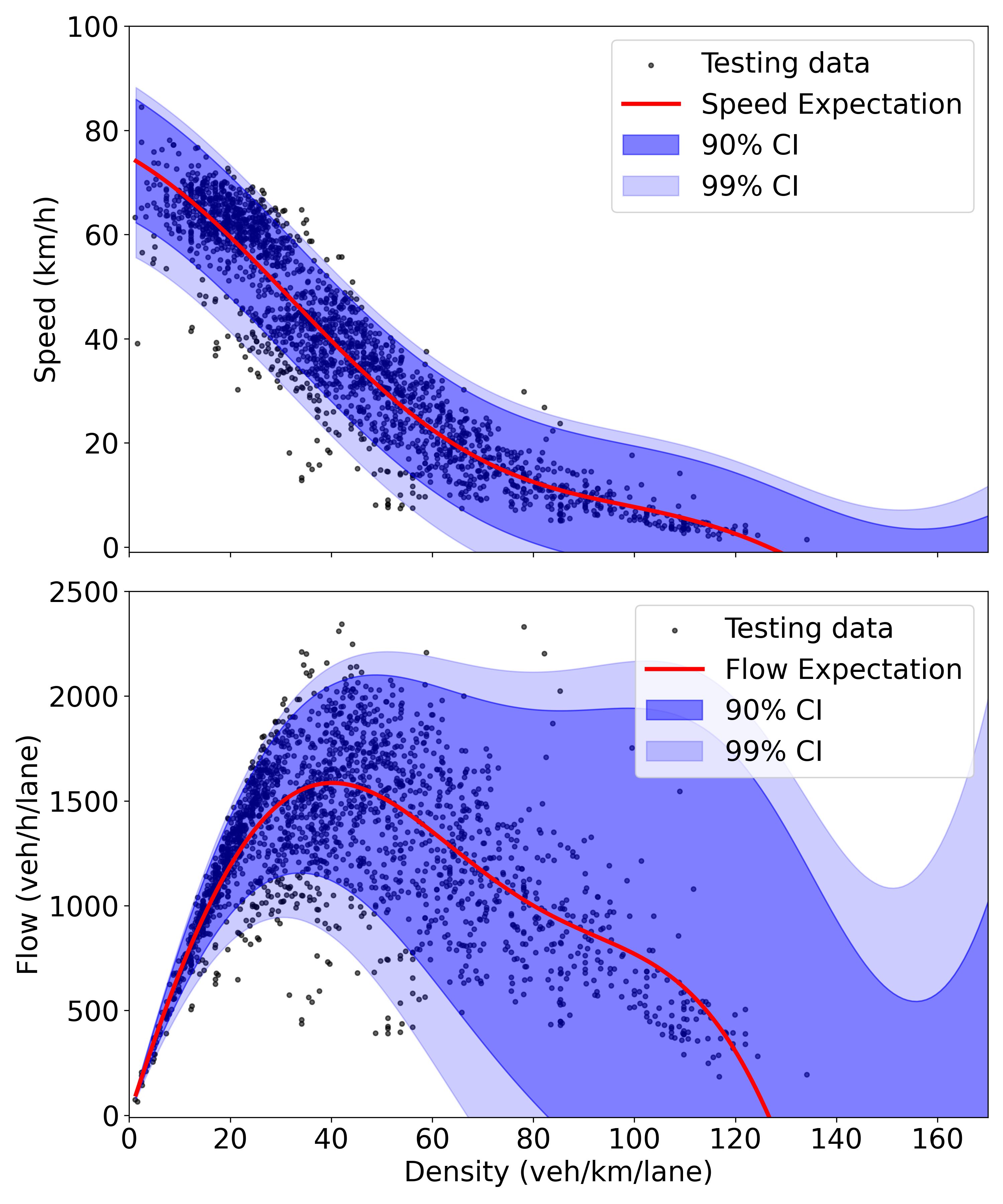}
        \caption{Greenshields Model with Gaussian Process }
        \label{subfig:gp-gs}
    \end{subfigure}
    \begin{subfigure}{0.45\textwidth}
        \includegraphics[width=\linewidth]{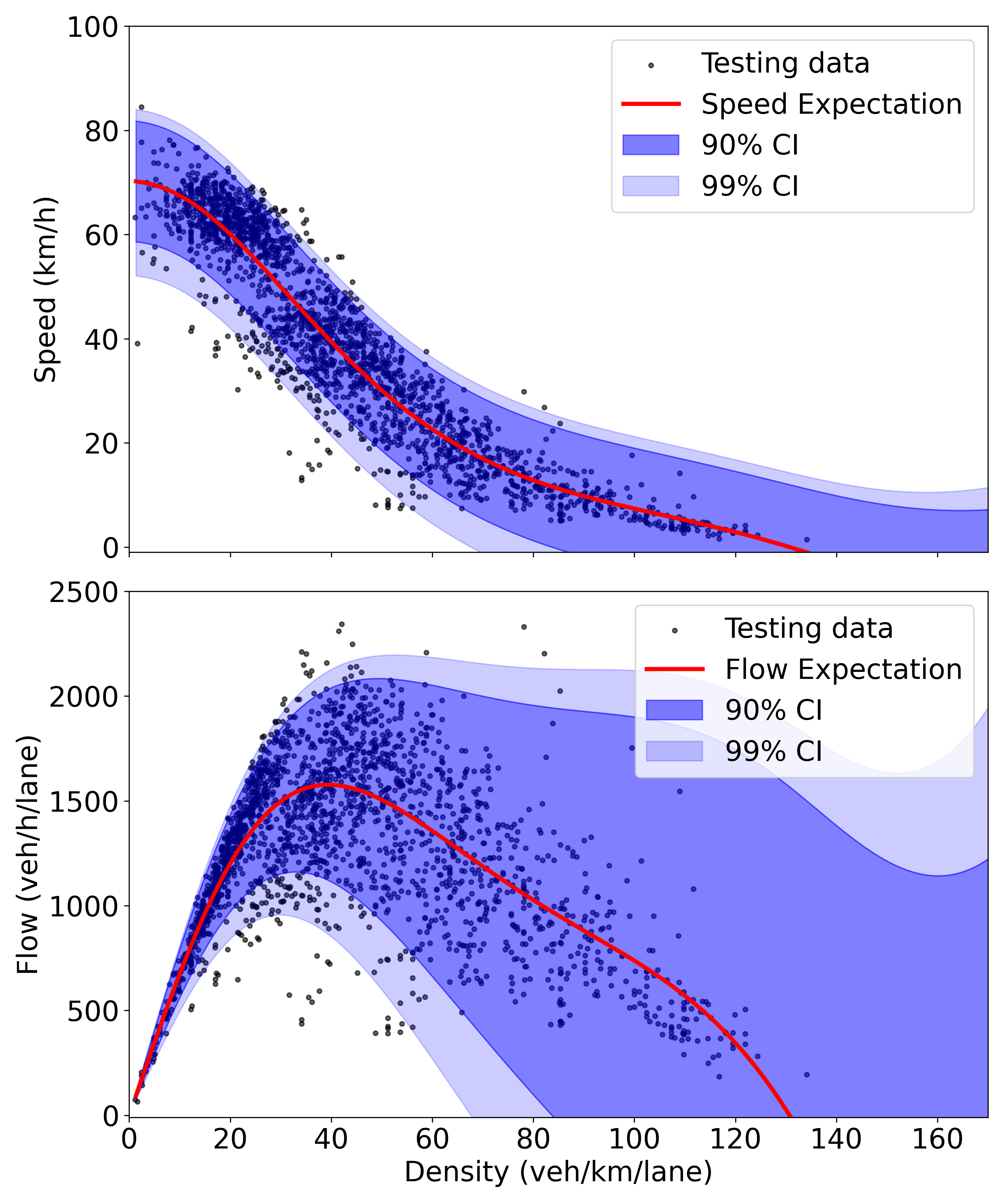}
        \caption{ S3 Model with Gaussian Process}
        \label{subfig:gp-s3}
    \end{subfigure}
    \vspace{0.5em}

    \begin{subfigure}{0.45\textwidth}
        \includegraphics[width=\linewidth]{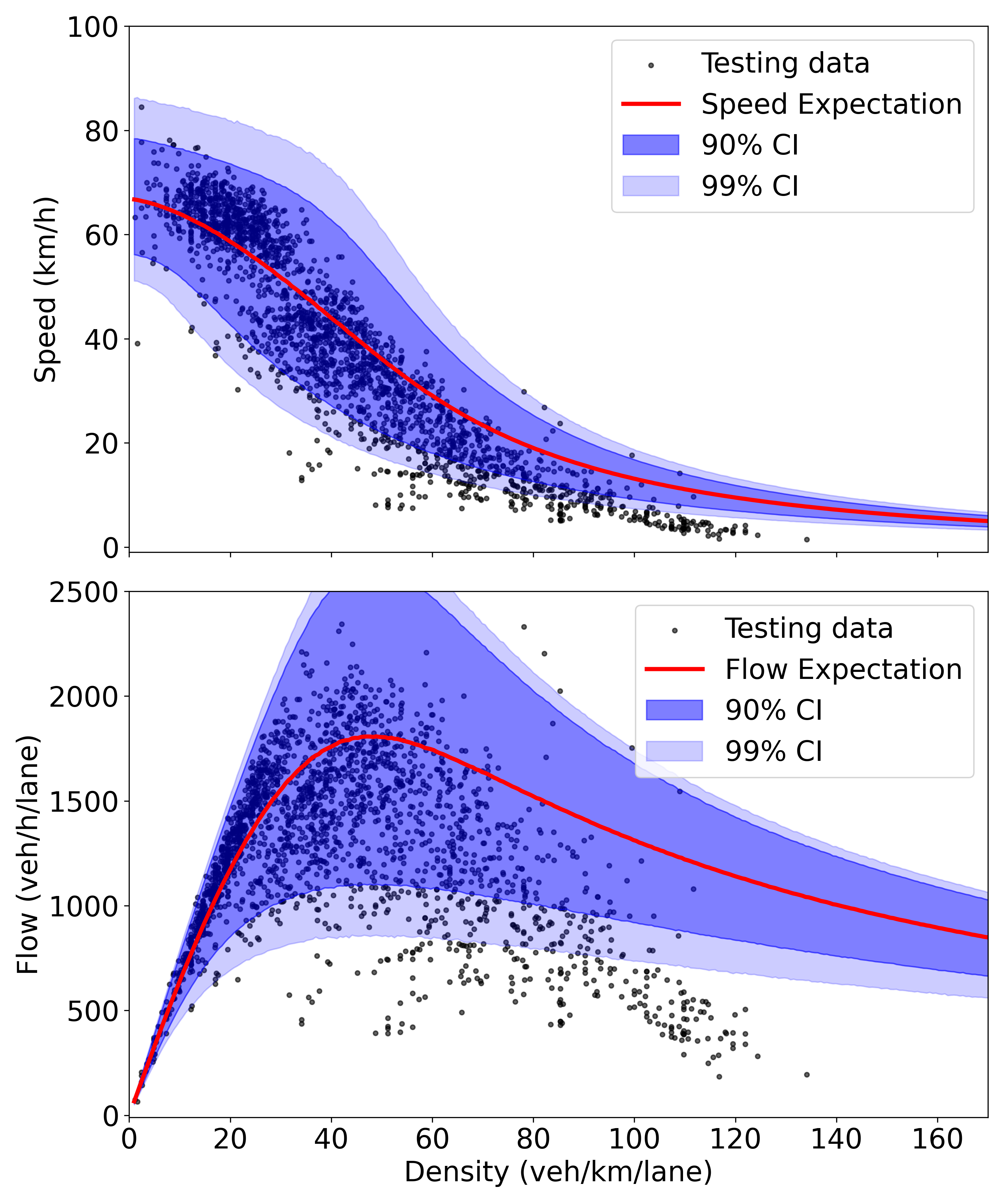}
        \caption{S3 Model with Log-Normal Distribution}
        \label{subfig:s3-lg}
    \end{subfigure}
    
    \caption{Modeling results from Baseline Stochastic FD Models}
    \label{fig:three_subfigs}
\end{figure}

Performance metrics for probabilistic estimation across all evaluated models are summarized in Table \ref{tab:estimation_performance}. 
Five evaluation metrics are reported for both speed-density and flow-density relations. The reported values represent statistical mean and standard deviation using 5-fold cross-validation. 
The NLL metric is omitted for the \textit{S3+LN} model because its probability density function does not admit a close-form expression.

Overall, the experimental results reveal that stochastic FD models developed under the proposed semiparametric framework consistently outperform the baseline approaches. Among all baseline models, the \textit{S3+LN} model exhibits the poorest performance across most of the metrics. In contrast, the Gaussian Process (GP) regression models achieve substantially better performance relative to the \textit{S3+LN} approach.

The table also reveals performance differences among the model variants. For the baseline models, the GP regression based on the deterministic S3 model outperforms the formulation based on the Greenshields model. This advantage is consistent across cross-validation folds, as reflected by the relatively small standard deviations. 

Among the models proposed in this study, semiparametric models employing the Skew-Normal (SN) distribution consistently outperform their Normal (N) distribution counterparts, indicating that modeling asymmetric uncertainty improves predictive performance. 
Furthermore, the BwNC formulation demonstrates slight better mean performance than the QwNC formulation. However, the considerable overlap in their standard deviations suggests that the performance difference between the two formulations is relatively modest. 

\subsection{Modeling Analysis}
Figure~\ref{fig:five_subfigs} illustrates stochastic FD modeling results from semiparametric models proposed in this study, while Figure~\ref{fig:three_subfigs} shows the corresponding results from baseline models.
In these visualizations, the red curve represents the conditional expected flow or speed for each density value, while the blue shading areas demonstrate the 90\% and 99\% confidence intervals. 


The semiparametric models produce physically consistent FD relationships. Specifically, both flow and speed approach zero when traffic density approaches the estimated jam density. The estimated jam density is approximately $150$ veh/km/lane, which is consistent with the typical observations in the traffic environment. Furthermore, the predicted uncertainty covers the majority of the observed data points and gradually shrink as density approaches the jam density in the congestion regime. 

Furthermore, models based on Skew-Normal distribution are able to capture the transition between free flow and congestion regimes, which occurs around a density of $35$ veh/km/lane. In addition, the SN-BwNC and N-QwNC models generate smooth functional relationships without noticeable overfitting, suggesting that the semiparametric formulation provides a good balance between model flexibility and structural regularization.

The results for the \textit{S3+LN} model reveal a pronounced bias toward higher values in heavy congestion regimes, as illustrated in Figure~\ref{subfig:s3-lg}. Specifically, in both the speed-density and flow-density relationships, the empirical observations in the congested region consistently fall outside the 99\% confidence interval of the model estimation. 

A comparison of model parameters between \textit{S3+LN} and the deterministic S3 model (Table~\ref{tab:prove_s3}) indicates that this performance degradation stems from the inherent inability of the deterministic model to characterize heavy congestion regimes. Treating parameters as random variables primarily serves to introduce uncertainty without substantially improving the underlying fitting quality of the deterministic framework. Consequently, \textit{S3+LN} inherits the structural drawbacks of the \textit{S3} model, leading to persistent systematic bias in higher density regions.

In contrast, Gaussian Process (GP) regression demonstrates a superior capacity for stochastic modeling, though it encounters challenges in generating consistent uncertainty across the entire density range. As shown in Figure~\ref{subfig:gp-s3}, the inherent bias of the \textit{S3} model is successfully mitigated in the \textit{S3+GP} configuration. Similarly, Figure~\ref{subfig:gp-gs} displays the results for \textit{S3+GP}, where the Gaussian Process effectively introduces necessary nonlinearity to the standard Greenshields speed--density model.
However, the performance in the heavy congestion region remains suboptimal. The scarcity of training data in high-density states leads to a significant localized increase in model uncertainty. This variance is further magnified when transformed into the flow-density relation, where the multiplicative effect of density results in a wide, undesirable predictive interval that limits the model's practical utility for congestion forecasting.

\begin{figure}[!t]
  \centering
    \includegraphics[width=\textwidth]{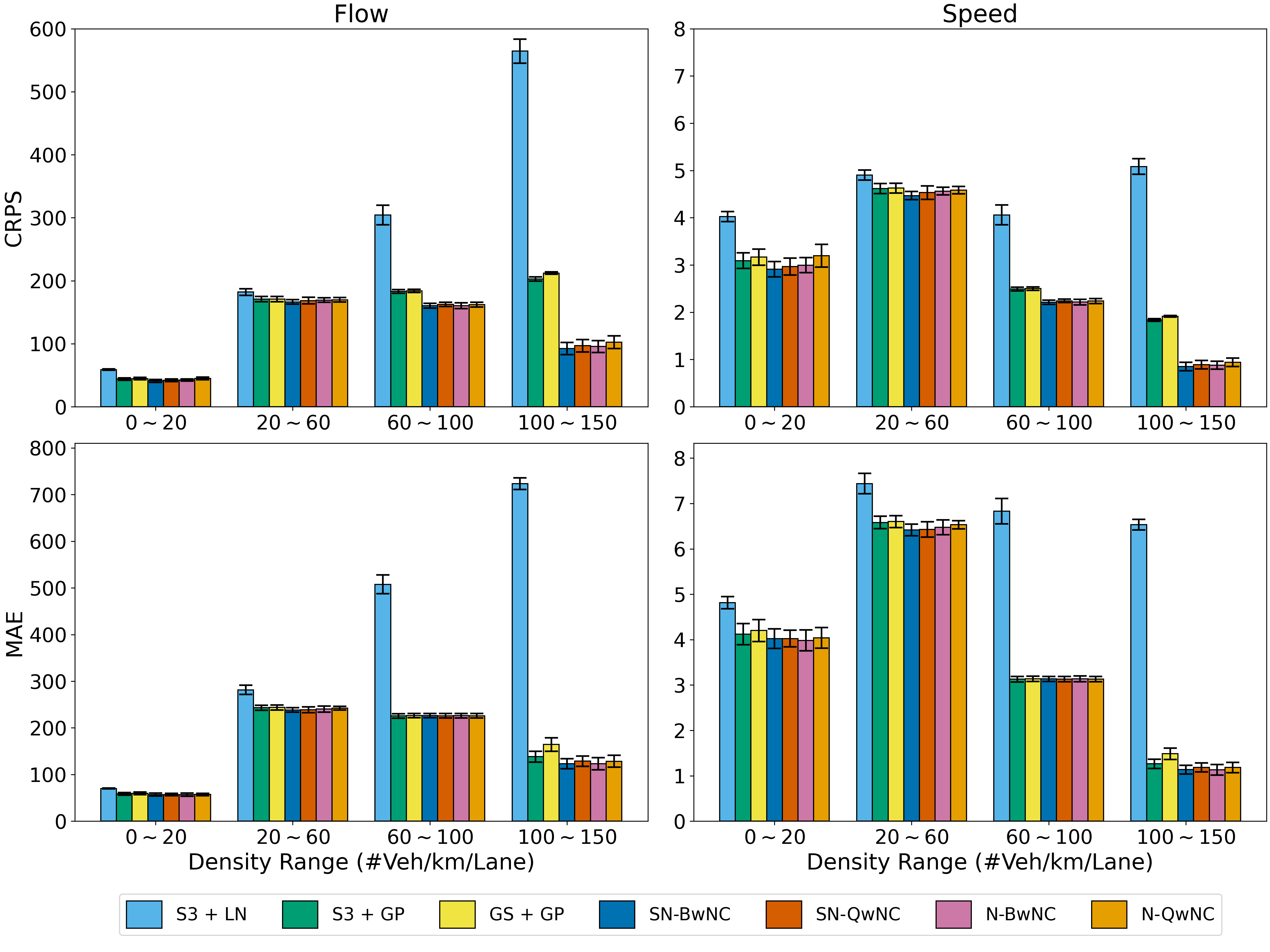} 
  \caption{CRPS and MAE Comparison on Different Density Regions}
  \label{fig:comparison}
\end{figure}

\subsection{Performance on Different Density Regions}
The CRPS and MAE comparison of the stochastic FD models under different density regions is presented in Figure \ref{fig:comparison}. The density is categorized into four regimes: free flow (0-20 veh/km), transition (20-60 veh/km), light congestion (60-100 veh/km), and heavy congestion (100-150 veh/km). 

Semiparametric models consistently outperform baseline counterparts across all density regimes. This improvement is most pronounced in the heavy congestion region, where baseline models typically struggle with high variance.

Regarding the flow-density relationship, a divergence in performance trends is observed as density increases. While the CRPS for baseline models degrades significantly in high-density states, the semiparametric models exhibit enhanced predictive accuracy and stability under heavy congestion.

For the speed-density relationship, all models follow a non-monotonic performance curve in the speed-density relationship, with the highest error rates (worst performance) appeared in the transition region, likely due to the inherent instability of traffic flow during phase shifts.


The deterministic performance of semiparametric SFD models is comparable to that of the Gaussian Process (GP) models. This suggests that the potential driver for improved stochastic modeling is the integration of physical constraints, which effectively regularizes the model's uncertainty in high-congestion regions.


\subsection{Ablation on Regularization Loss}
\begin{figure}[!tb]
    \centering
    \includegraphics[width=\linewidth]{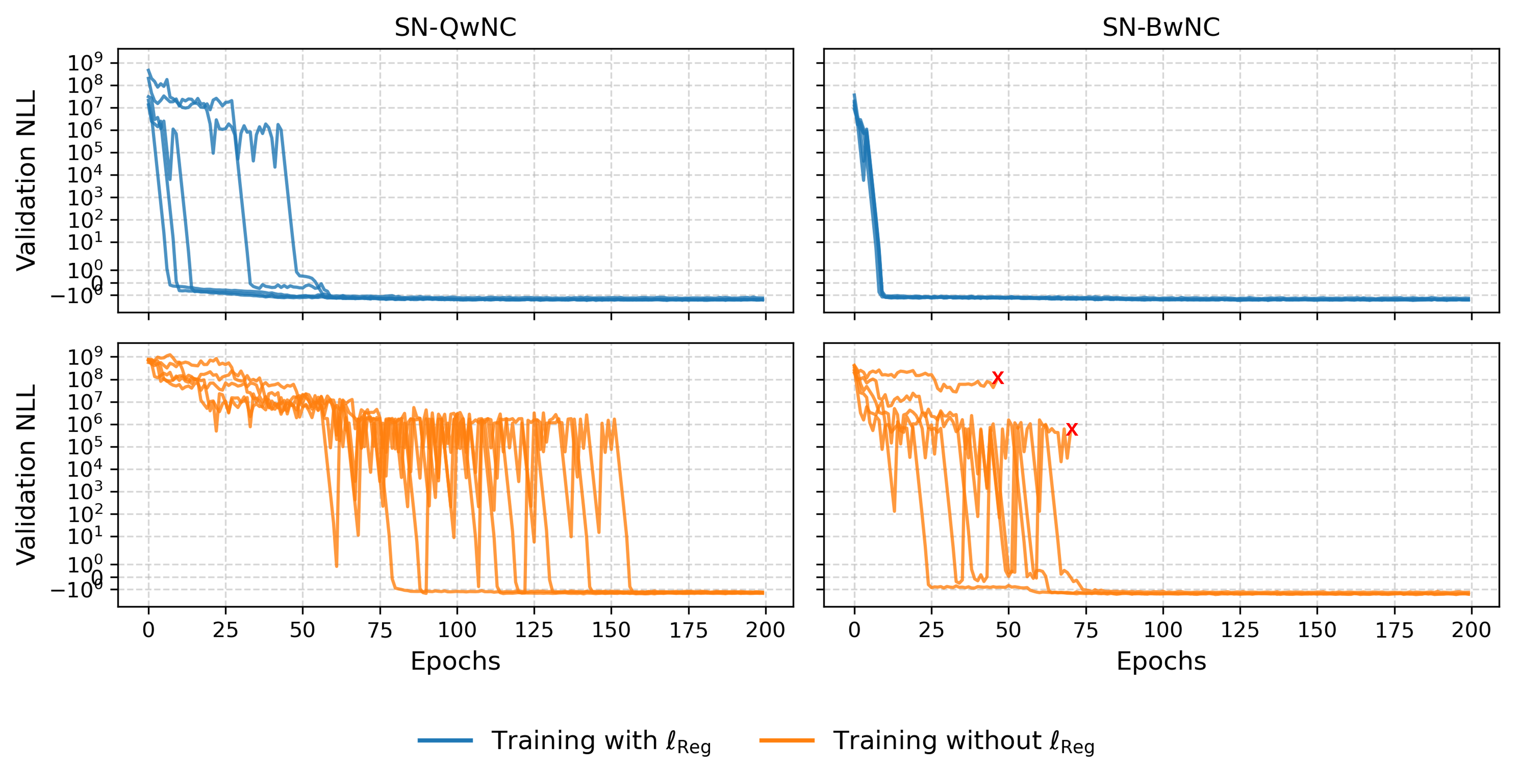}
    \caption{The convergence patterns during model training with v.s. without regularization loss in the objective function.}
    \label{fig:ablation-reg}
\end{figure}

The influence of the regularization term $\ell_{\mathrm{Reg}}$ on training stability is evaluated in Figure~\ref{fig:ablation-reg}, which depicts the negative log-likelihood (NLL) loss across five cross-validation folds. The blue and orange curves represent the training trajectories with and without the inclusion of $\ell_{\mathrm{Reg}}$, respectively.

The experimental results indicate that $\ell_{\mathrm{Reg}}$ is indispensable for numerical stability in both the SN-QwNC and SN-BwNC architectures. In the absence of this regularization term, both models exhibit highly volatile loss curves and a significant reduction in convergence speed. Notably, the SN-BwNC model fails to achieve convergence in two out of the five folds when $\ell_{\mathrm{Reg}}$ is removed. These ablation results confirm that the proposed regularization loss is a critical component for ensuring both the stability and the overall effectiveness of the stochastic training process.

\section{Discussion}\label{sec:discussion}
The effectiveness of the proposed framework relies on its ability to
propagate first and second-moment constraints to the distribution parameters, specifically the constrained parameters $\bm\eta_\mathrm{con}$. 
Therefore, after determining constrained parameters, the primary theoretical challenge lies in the existence of a solution for the moment-matching system in Eq.~\eqref{eq:moment_matching}.
For a distribution $f$ and values $0 < m, s < \infty$, the problem simplifies to:
\begin{equation}\label{eq:simplified}
\begin{cases}
    \ \mathbb{E}_{X\sim f}[X] &= m\\
   \ \mathrm{Var}_{X\sim f}[X] &= s^2\\
\end{cases}
\end{equation}

As established in Proposition~\ref{prop:moment_matching}, any
distribution belonging to a location-scale family guarantees a
unique solution to this system. However, the applicability of our framework extends beyond this family: while some distributions naturally yield unique solutions, others become feasible once appropriate parameter constraints are imposed.

\subsection{Feasible Distribution from Non-Location-scale Family}
\subsubsection{Gamma Distribution}
A primary example of a feasible non-location-scale distribution is the Gamma distribution, with probability density function as:
\begin{equation}
    f(x; \alpha, \theta) = \frac{x^{\alpha-1}e^{-x/\theta}}{\theta^\alpha \Gamma(\alpha)}, \quad \text{for } x > 0,
\end{equation}
where $\Gamma(\cdot)$ is the gamma function, and $\{\alpha,\theta\}$ are positive values.  The first two moments are defined as:
\begin{equation}
    \mathbb{E}[X] = \alpha\theta, \qquad \mathrm{Var}(X) = \alpha\theta^2.
\end{equation}
By defining the constrained parameters as $\bm\eta_\mathrm{con}=[\alpha, \ \theta]^\mathsf T$,  we obtain a unique solution for the system of equations in Eq.~\eqref{eq:simplified}:
\begin{equation}
    \theta = \frac{s^2}{m}, \qquad \alpha = \frac{m^2}{s^2}.
\end{equation}

\subsubsection{Log-Normal Distribution}
The Log-Normal distribution is another widely used non-location-scale family that fits our framework. 
With its probability density function defined by
\begin{equation}
    f(x; \mu, \sigma) = \frac{1}{x\sigma\sqrt{2\pi}} \exp\left(-\frac{(\ln x - \mu)^2}{2\sigma^2}\right), \quad \text{for } x > 0,
\end{equation}
where $\mu \in \mathbb{R}$  and $\sigma > 0$, its first two moments are:
\begin{equation}
    \mathbb E[X]=e^{\mu+\sigma^2/2},\quad 
    \mathrm{Var}(X)=(e^{\sigma^2}-1)e^{2\mu+\sigma^2}.
\end{equation}
Here, the constrained parameters are defined as $\bm\eta_\mathrm{con}=[\mu, \ \sigma]^\mathsf T$. To adapt this to our framework, we substitute these theoretical moments into Eq.~\eqref{eq:simplified}. This yields a unique, closed-form solution for the parameters:
\begin{equation}
\begin{aligned}
    \sigma=\left(\ln\!\left(1+\frac{s^2}{m^2}\right)\right)^{\frac{1}{2}},
\qquad
\mu=\ln m-\frac{1}{2}\ln\!\left(1+\frac{s^2}{m^2}\right).
\end{aligned}
\end{equation}

\subsection{Distributions Requiring Additional Constraints}
\subsubsection{Truncated Distributions}
While many distributions inherently support the moment-matching system, certain families, such as truncated distributions, require additional mathematical constraints to remain feasible. 
As demonstrated by \citet{bhatia2000better}, the variance of any bounded probability distribution is strictly limited by its expectation, as well as its absolute lower bound $L$ and upper bound $U$. 
This relationship is expressed as:
\begin{equation}\label{eq:bhatia_bound}
    \mathrm{Var}[X] \le (\mathbb{E}[X]-L) (U-\mathbb{E}[X]).
\end{equation}
Consequently, the condition $0<s<\infty$ does not guarantee a solution for the system in Eq.~\eqref{eq:simplified}, particularly when $s$ is excessively large. 
Translating the mathematical limitation into our physical constraints, an upper bound is needed for the standard deviation:
\begin{equation}
\begin{cases}
\mathbb{E}_{q \sim p(q \mid \rho)}[q] = \mathbb{S}_{q \sim p(q \mid \rho)}[q] = 0, & \rho \in \{0, \rho_{\mathrm{jam}}\}, \\
0 < \mathbb{E}_{q \sim p(q \mid \rho)}[q] < \infty, & \rho \in (0, \rho_{\mathrm{jam}}), \\
0 < \mathbb{S}_{q \sim p(q \mid \rho)}[q] \le S_{\max}(\rho), & \rho \in (0, \rho_{\mathrm{jam}}).
\end{cases}
\end{equation}
 $S_{\max}(\rho)$ represents its theoretical upper bound derived from Eq.~\eqref{eq:bhatia_bound}, given by:
 \begin{equation}
    S_{\max}(\rho) = \sqrt{ \left(\mathbb{E}_{q \sim p(q \mid \rho)}[q] - L(\rho)\right) \left(U(\rho) - \mathbb{E}_{q \sim p(q \mid \rho)}[q]\right) },
\end{equation}
where $L(\rho)$ and $U(\rho)$ denote the physical lower bound (e.g., minimum flow) and upper bound (i.e., maximum capacity) of flow $q$ at density $\rho$, respectively. 
To guarantee the existence of a valid solution for the moment-matching equations, the functional design of the standard deviation $s(\rho)$ must explicitly depend on the expectation function $m(\rho)$, as well as the local boundary functions $L(\rho)$ and $U(\rho)$.

\subsubsection{Gaussian Mixture Model}
Another complex scenario arises when employing mixture distributions, such as a two-component Gaussian Mixture Model (GMM). The probability density function is given by:
\begin{equation}
p(x) = k\mathcal N(x; \mu_1,\sigma_1^2) + (1-k)\mathcal N(x; \mu_2,\sigma_2^2),
\end{equation}
where $k \in (0,1)$ is the mixing weight. The theoretical mean and variance can be analytically expressed as:
\begin{equation}
\begin{aligned}
    \mathbb E[X] &= k\mu_1+(1-k)\mu_2,\\
    \mathrm{Var}(X) &= k\sigma_1^2 + (1-k)\sigma_2^2 + k(1-k)(\mu_1 - \mu_2)^2.
\end{aligned}
\end{equation}
The parameters are partitioned into constrained parameters $\bm\eta_\mathrm{con} = [\mu_1,\ \mu_2]^\mathsf T$, and unconstrained  parameters $\bm\eta_\mathrm{free} = [k,\ \sigma_1,\ \sigma_2]^\mathsf T$. Equating the moments to $m$ and $s^2$ provides the following symmetric closed-form solutions:
\begin{equation}
\mu_1 = m \pm \sqrt{\frac{1-k}{k}\Delta}, \qquad\mu_2 = m \mp \sqrt{\frac{k}{1-k}\Delta}\ ,
\end{equation}
where $\Delta = s^2 - \left[ k\sigma_1^2 + (1-k)\sigma_2^2 \right] > 0$. A real solution necessitates $\Delta > 0$, thereby introducing a lower bound for the standard deviation. Consequently, the physical constraints must be adjusted:

\begin{equation}
\begin{cases}
\mathbb{E}_{q \sim p(q \mid \rho)}[q] = \mathbb{S}_{q \sim p(q \mid \rho)}[q] = 0, & \rho \in \{0, \rho_{\mathrm{jam}}\}, \\
0 < \mathbb{E}_{q \sim p(q \mid \rho)}[q] < \infty, & \rho \in (0, \rho_{\mathrm{jam}}), \\
S_{\min}(\rho) < \mathbb{S}_{q \sim p(q \mid \rho)}[q] < \infty, & \rho \in (0, \rho_{\mathrm{jam}}).
\end{cases}
\end{equation}
where $S_{\min}(\rho) = \sqrt{k(\rho)\sigma_1^2(\rho) + (1-k(\rho))\sigma_2^2(\rho)}$. Ultimately, the design of the standard deviation function $s(\rho)$ must carefully accommodate the selected free parameters $\bm\eta_\mathrm{free}$ to ensure the condition $\Delta > 0$ is continuously satisfied.

\section{Conclusion}
In this study, we proposed a novel semiparametric framework for modeling stochastic fundamental diagrams. By partitioning distribution parameters into a constrained component and a data-driven neural component, our framework reconciles physical consistency with empirical flexibility. This semiparametric design ensures that essential traffic properties, such as boundary conditions at zero and jam density, are satisfied by construction while capturing the complex stochastic patterns inherent in traffic data.

From a theoretical perspective, we established that the underlying moment-matching system admits a unique solution for location–scale families, guaranteeing the well-posedness of the parameterization. We further demonstrated the framework’s extensibility by deriving feasibility conditions for non-location–scale families, including truncated and mixture-based distributions in the final discussion.

Empirical evaluation using the MAGIC dataset demonstrates that our semiparametric approach consistently outperforms baseline models. Specifically, the proposed models provide superior probabilistic accuracy and robust uncertainty quantification, particularly in congested regimes where the conventional stochastic extension of deterministic models often exhibit systematic bias. Our results further highlight that incorporating asymmetric distributions significantly enhances the model's ability to capture observed empirical variance.

In summary, this framework provides a flexible, theoretically grounded foundation for stochastic FD modeling. Future research will explore the integration of diverse distribution families, the extension of the model to multi-regime and multi-lane dynamics, and the inclusion of spatiotemporal dependencies to further improve the fidelity of stochastic traffic flow representations.

\section*{Acknowledgment}
 The authors thank the Swedish Transport Administration for supporting this work. 

\appendix
\section{Parameters for the Log-Normal-Based S3 Model}
\label{app1}

Table \ref{tab:prove_s3} shows marginal difference between the \textit{S3} model parameters obtained using the original method \citep{cheng2021s} (i.e., the weighted least squares method) and those derived from the method by \citet{cheng2024analytical}.

\begin{table}[!htbp]
\centering
\small
\setlength{\tabcolsep}{4pt}
\caption{Parameter Value Comparison between Deterministic Model and Log-Normal-Based Stochastic Model}
\label{tab:prove_s3}
\begin{tabular}{lccc}
\toprule
& S3 + WLSM & Expected Parameter Value & Parameter as Random Variable \\
\midrule
$v_{\mathrm{free}}$
& $67.773_{\pm0.259}$
& $66.290_{\pm0.205}$
& $\ln(v_{\mathrm{free}})\sim \mathcal{N}(\mu=4.194_{\pm0.003},\sigma=0.104_{\pm0.004})$ \\

$v_{\mathrm{critical}}$
& $41.008_{\pm0.238}$
& $37.887_{\pm0.057}$
& $\ln(v_{\mathrm{critical}})\sim \mathcal{N}(\mu=3.635_{\pm0.002},\sigma=0.265_{\pm0.003})$ \\

$m$
& $2.760_{\pm0.039}$
& $2.478_{\pm0.017}$
& $\tfrac{1}{m}\sim \mathcal{N}(\mu=0.404_{\pm0.003},\sigma=0.205_{\pm0.001})$ \\

$\rho_{\mathrm{critical}}$
& $48.554_{\pm0.170}$
& $48.554_{\pm0.170}$
& Constant \\
\bottomrule
\end{tabular}
\end{table}

\newpage
\bibliographystyle{elsarticle-harv}
\bibliography{cas-refs}

@inproceedings{greenshields1935study,
  title={A study of traffic capacity},
  author={Greenshields, Bruce D. and Bibbins, J. Rowland and Channing, W. S. and Miller, Harvey H.},
  booktitle={{Highway Research Board Proceedings}},
  volume={14},
  number={1},
  pages={448--477},
  year={1935},
  address={Washington, D.C.}
}

@article{greenberg1959analysis,
  title={An analysis of traffic flow},
  author={Greenberg, Harold},
  journal={Operations Research},
  volume={7},
  number={1},
  pages={79--85},
  year={1959},
  publisher={INFORMS}
}

@article{newell1961nonlinear,
  title={Nonlinear effects in the dynamics of car following},
  author={Newell, Gordon Frank},
  journal={Operations Research},
  volume={9},
  number={2},
  pages={209--229},
  year={1961},
  publisher={INFORMS}
}

@article{ma2022magic,
  title={{MAGIC} dataset: Multiple conditions unmanned aerial vehicle group-based high-fidelity comprehensive vehicle trajectory dataset},
  author={Ma, Wanjing and Zhong, Hao and Wang, Ling and Jiang, Linzhi and Abdel-Aty, Mohamed},
  journal={Transportation Research Record},
  volume={2676},
  number={5},
  pages={793--805},
  year={2022},
  publisher={SAGE Publications},
  address={Los Angeles, CA}
}

@book{ni2015traffic,
  title={Traffic flow theory: Characteristics, experimental methods, and numerical techniques},
  author={Ni, Daiheng},
  year={2015},
  publisher={Butterworth-Heinemann}
}

@article{cheng2024analytical,
  title={Analytical formulation for explaining the variations in traffic states: A fundamental diagram modeling perspective with stochastic parameters},
  author={Cheng, Qixiu and Lin, Yuqian and Zhou, Xuesong Simon and Liu, Zhiyuan},
  journal={European Journal of Operational Research},
  volume={312},
  number={1},
  pages={182--197},
  year={2024},
  publisher={Elsevier}
}

@article{lei2024unraveling,
  title={Unraveling stochastic fundamental diagrams with empirical knowledge: Modeling, limitations, and future directions},
  author={Lei, Yuan-Zheng and Gong, Yaobang and Yang, Xianfeng Terry},
  journal={Transportation Research Part C: Emerging Technologies},
  volume={169},
  pages={104851},
  year={2024},
  publisher={Elsevier}
}

@article{qu2015fundamental,
  title={On the fundamental diagram for freeway traffic: A novel calibration approach for single-regime models},
  author={Qu, Xiaobo and Wang, Shuaian and Zhang, Jin},
  journal={Transportation Research Part B: Methodological},
  volume={73},
  pages={91--102},
  year={2015},
  publisher={Elsevier}
}

@article{cheng2021s,
  title={An {S}-shaped three-parameter {(S3)} traffic stream model with consistent car following relationship},
  author={Cheng, Qixiu and Liu, Zhiyuan and Lin, Yuqian and Zhou, Xuesong Simon},
  journal={Transportation Research Part B: Methodological},
  volume={153},
  pages={246--271},
  year={2021},
  publisher={Elsevier}
}

@article{jayakrishnan1995dynamic,
  title={A dynamic traffic assignment model with traffic-flow relationships},
  author={Jayakrishnan, R. and Tsai, Wei K. and Chen, Anthony},
  journal={Transportation Research Part C: Emerging Technologies},
  volume={3},
  number={1},
  pages={51--72},
  year={1995},
  publisher={Elsevier}
}

@article{wang2011logistic,
  title={Logistic modeling of the equilibrium speed--density relationship},
  author={Wang, Haizhong and Li, Jia and Chen, Qian-Yong and Ni, Daiheng},
  journal={Transportation Research Part A: Policy and Practice},
  volume={45},
  number={6},
  pages={554--566},
  year={2011},
  publisher={Elsevier}
}

@article{papageorgiou1989macroscopic,
  title={Macroscopic modelling of traffic flow on the {Boulevard P{\'e}riph{\'e}rique} in {Paris}},
  author={Papageorgiou, Markos and Blosseville, Jean-Marc and Hadj-Salem, Habib},
  journal={Transportation Research Part B: Methodological},
  volume={23},
  number={1},
  pages={29--47},
  year={1989},
  publisher={Elsevier}
}

@article{del1995functional,
  title={On the functional form of the speed-density relationship---{II}: Empirical investigation},
  author={Del Castillo, J. M. and Benitez, F. G.},
  journal={Transportation Research Part B: Methodological},
  volume={29},
  number={5},
  pages={391--406},
  year={1995},
  publisher={Elsevier}
}

@article{bai2021calibration,
  title={Calibration of stochastic link-based fundamental diagram with explicit consideration of speed heterogeneity},
  author={Bai, Lu and Wong, S. C. and Xu, Pengpeng and Chow, Andy H. F. and Lam, William H. K.},
  journal={Transportation Research Part B: Methodological},
  volume={150},
  pages={524--539},
  year={2021},
  publisher={Elsevier}
}

@article{jabari2014probabilistic,
  title={A probabilistic stationary speed--density relation based on {Newell's} simplified car-following model},
  author={Jabari, Saif Eddin and Zheng, Jianfeng and Liu, Henry X.},
  journal={Transportation Research Part B: Methodological},
  volume={68},
  pages={205--223},
  year={2014},
  publisher={Elsevier}
}

@article{ni2018modeling,
  title={Modeling phase diagrams as stochastic processes with application in vehicular traffic flow},
  author={Ni, Daiheng and Hsieh, Hui K. and Jiang, Tao},
  journal={Applied Mathematical Modelling},
  volume={53},
  pages={106--117},
  year={2018},
  publisher={Elsevier}
}

@article{qu2017stochastic,
  title={On the stochastic fundamental diagram for freeway traffic: Model development, analytical properties, validation, and extensive applications},
  author={Qu, Xiaobo and Zhang, Jin and Wang, Shuaian},
  journal={Transportation Research Part B: Methodological},
  volume={104},
  pages={256--271},
  year={2017},
  publisher={Elsevier}
}

@article{wang2021model,
  title={Model on empirically calibrating stochastic traffic flow fundamental diagram},
  author={Wang, Shuaian and Chen, Xinyuan and Qu, Xiaobo},
  journal={Communications in Transportation Research},
  volume={1},
  pages={100015},
  year={2021},
  publisher={Elsevier}
}

@article{bramich2023fitfun,
  title={{FitFun}: A modelling framework for successfully capturing the functional form and noise of observed traffic flow--density--speed relationships},
  author={Bramich, Daniel M. and Menendez, Monica and Amb{\"u}hl, Lukas},
  journal={Transportation Research Part C: Emerging Technologies},
  volume={151},
  pages={104068},
  year={2023},
  publisher={Elsevier}
}

@article{ahmed2021fundamental,
  title={On the fundamental diagram and driving behavior modeling of heterogeneous traffic flow using {UAV}-based data},
  author={Ahmed, Afzal and Ngoduy, Dong and Adnan, Muhammad and Baig, Mirza Asad Ullah},
  journal={Transportation Research Part A: Policy and Practice},
  volume={148},
  pages={100--115},
  year={2021},
  publisher={Elsevier}
}

@article{zhou2020modeling,
  title={Modeling the fundamental diagram of mixed human-driven and connected automated vehicles},
  author={Zhou, Jiazu and Zhu, Feng},
  journal={Transportation Research Part C: Emerging Technologies},
  volume={115},
  pages={102614},
  year={2020},
  publisher={Elsevier}
}

@article{hendrycks2016gaussian,
  title={{Gaussian} error linear units {(GELUs)}},
  author={Hendrycks, Dan and Gimpel, Kevin},
  journal={arXiv preprint arXiv:1606.08415},
  year={2016}
}

@article{kingma2014adam,
  title={{Adam}: A method for stochastic optimization},
  author={Kingma, Diederik P. and Ba, Jimmy},
  journal={arXiv preprint arXiv:1412.6980},
  year={2014}
}

@article{zhang2025stochastic,
  title={On the stochastic fundamental diagram: A general micro-macroscopic traffic flow modeling framework},
  author={Zhang, Xiaohui and Sun, Jie and Sun, Jian},
  journal={Communications in Transportation Research},
  volume={5},
  pages={100163},
  year={2025},
  publisher={Elsevier}
}

@article{paszke2019pytorch,
  title={{PyTorch}: An imperative style, high-performance deep learning library},
  author={Paszke, Adam and Gross, Sam and Massa, Francisco and Lerer, Adam and Bradbury, James and Chanan, Gregory and Killeen, Trevor and Lin, Zeming and Gimelshein, Natalia and Antiga, Luca and others},
  journal={Advances in Neural Information Processing Systems},
  volume={32},
  year={2019}
}

@inproceedings{frostig2019compiling,
  title={Compiling machine learning programs via high-level tracing},
  author={Frostig, Roy and Johnson, Matthew James and Leary, Chris},
  booktitle={{SysML} Conference 2018},
  year={2019}
}

@article{storm2022efficient,
  title={Efficient evaluation of stochastic traffic flow models using {Gaussian} process approximation},
  author={Storm, Pieter Jacob and Mandjes, Michel and van Arem, Bart},
  journal={Transportation Research Part B: Methodological},
  volume={164},
  pages={126--144},
  year={2022},
  publisher={Elsevier}
}

@article{hornik1989multilayer,
  title={Multilayer feedforward networks are universal approximators},
  author={Hornik, Kurt and Stinchcombe, Maxwell and White, Halbert},
  journal={Neural Networks},
  volume={2},
  number={5},
  pages={359--366},
  year={1989},
  publisher={Elsevier}
}

@article{ni2016vehicle,
  title={Vehicle longitudinal control and traffic stream modeling},
  author={Ni, Daiheng and Leonard, John D. and Jia, Chaoqun and Wang, Jianqiang},
  journal={Transportation Science},
  volume={50},
  number={3},
  pages={1016--1031},
  year={2016},
  publisher={INFORMS}
}

@article{liu2023gaussian,
  title={A {Gaussian}-process-based data-driven traffic flow model and its application in road capacity analysis},
  author={Liu, Zhiyuan and Lyu, Cheng and Wang, Zelin and Wang, Shuaian and Liu, Pan and Meng, Qiang},
  journal={IEEE Transactions on Intelligent Transportation Systems},
  volume={24},
  number={2},
  pages={1544--1563},
  year={2023},
  publisher={IEEE}
}

@article{rigby2005generalized,
  title={Generalized additive models for location, scale and shape},
  author={Rigby, Robert A. and Stasinopoulos, D. Mikis},
  journal={Journal of the Royal Statistical Society Series C: Applied Statistics},
  volume={54},
  number={3},
  pages={507--554},
  year={2005},
  publisher={Oxford University Press}
}

@article{ngoduy2011multiclass,
  title={Multiclass first-order traffic model using stochastic fundamental diagrams},
  author={Ngoduy, D.},
  journal={Transportmetrica},
  volume={7},
  number={2},
  pages={111--125},
  year={2011},
  publisher={Taylor \& Francis}
}

@article{shi2021physics,
  title={A physics-informed deep learning paradigm for traffic state and fundamental diagram estimation},
  author={Shi, Rongye and Mo, Zhaobin and Huang, Kuang and Di, Xuan and Du, Qiang},
  journal={IEEE Transactions on Intelligent Transportation Systems},
  volume={23},
  number={8},
  pages={11688--11698},
  year={2021},
  publisher={IEEE}
}

@article{cheng2022bayesian,
  title={{Bayesian} calibration of traffic flow fundamental diagrams using {Gaussian} processes},
  author={Cheng, Zhanhong and Wang, Xudong and Chen, Xinyuan and Tr{\'e}panier, Martin and Sun, Lijun},
  journal={IEEE Open Journal of Intelligent Transportation Systems},
  volume={3},
  pages={763--771},
  year={2022},
  publisher={IEEE}
}

@techreport{bishop1994mixture,
  title={Mixture density networks},
  author={Bishop, Christopher M.},
  year={1994},
  institution={Aston University}
}

@article{theis2015generative,
  title={Generative image modeling using spatial {LSTMs}},
  author={Theis, Lucas and Bethge, Matthias},
  journal={Advances in Neural Information Processing Systems},
  volume={28},
  year={2015}
}

@article{gneiting2007strictly,
  title={Strictly proper scoring rules, prediction, and estimation},
  author={Gneiting, Tilmann and Raftery, Adrian E.},
  journal={Journal of the American Statistical Association},
  volume={102},
  number={477},
  pages={359--378},
  year={2007},
  publisher={Taylor \& Francis},
  doi={10.1198/016214506000001437}
}

@article{bhatia2000better,
  title={A better bound on the variance},
  author={Bhatia, Rajendra and Davis, Chandler},
  journal={The American Mathematical Monthly},
  volume={107},
  number={4},
  pages={353--357},
  year={2000},
  publisher={Taylor \& Francis}
}

\end{document}